\pdfoutput=1
\documentclass{article}
\usepackage{arxiv}
\usepackage[toc,page]{appendix}

\usepackage[usenames]{xcolor}

\usepackage{inputenc} %
\usepackage[T1]{fontenc}    %
\usepackage{hyperref}       %
\usepackage{url}            %
\usepackage{booktabs}       %
\usepackage{amsfonts}       %
\usepackage{nicefrac}       %
\usepackage{microtype}      %
\usepackage{lipsum}
\usepackage{mathtools}
\usepackage{cite}
\usepackage{amsmath}
\usepackage{amssymb}
\usepackage{amsthm}
\usepackage{algorithm,algcompatible}
\usepackage[toc,page]{appendix}

\usepackage{lineno}

\newtheorem{theorem}{Theorem}
\newtheorem{lemma}{Lemma}
\newtheorem{definition}{Definition}

\newtheorem{proposition}{Proposition}

\numberwithin{equation}{section}

\DeclareMathOperator{\sign}{sgn}

\usepackage{mathptmx}      %
\begin{document}

\title{Optimal Approximation of Zonoids and Uniform Approximation by Shallow Neural Networks
}

\author{Jonathan W. Siegel \\
 Department of Mathematics\\
 Texas A\&M University\\
 College Station, TX 77843 \\
 \texttt{jwsiegel@tamu.edu} \\
}

\maketitle

\begin{abstract}
    We study the following two related problems. The first is to determine to what error an arbitrary zonoid in $\mathbb{R}^{d+1}$ can be approximated in the Hausdorff distance by a sum of $n$ line segments. The second is to determine optimal approximation rates in the uniform norm for shallow ReLU$^k$ neural networks on their variation spaces. The first of these problems has been solved for $d\neq 2,3$, but when $d=2,3$ a logarithmic gap between the best upper and lower bounds remains. We close this gap, which completes the solution in all dimensions. For the second problem, our techniques significantly improve upon existing approximation rates when $k\geq 1$, and enable uniform approximation of both the target function and its derivatives.
\end{abstract}

\section{Introduction}
A (centered) zonotope in $\mathbb{R}^{d+1}$ (so that the sphere $S^d\subset \mathbb{R}^{d+1}$ is of dimension $d$) is a convex polytope $P$ which is the Minkowski sum of finitely many centered line segments, i.e., a body of the form
\begin{equation}\label{zonotope-P}
    P = \{x_1v_1 + \cdots + x_nv_n,~x_i\in[-1,1]\}
\end{equation}
for some collection of vectors $v_i\in \mathbb{R}^{d+1}$. The number $n$ is the number of summands of the zonotope. A zonoid is a convex body which is a limit of zonotopes in the Hausdorff metric. 

We consider the following problem: Given an arbitrary zonoid $Z$, how accurately can $Z$ be approximated by a polytope $P$ with $n$-summands? Here accuracy $\epsilon$ is taken to mean that $Z\subset P \subset (1+\epsilon)Z$. 

This problem has been studied by a variety of authors (see for instance \cite{bourgain1989approximation,bourgain1988distribution,linhart1989approximation,bourgain1993approximating,betke1983estimating,joos2023isoperimetric}). Of particular interest is the case when $Z = B^{d+1}$ is the Euclidean unit ball. In this case the problem has an equivalent formulation as (see \cite{betke1983estimating}): how many directions $v_1,...,v_n\in S^d$ are required to estimate the surface area of a convex body in $\mathbb{R}^{d+1}$ from the volumes of its $d$-dimensional projections orthogonal to each $v_i$?

In \cite{bourgain1989approximation}, it was shown using spherical harmonics that with $n$ summands and $Z = B^{d+1}$ the best error one can achieve is lower bounded by
\begin{equation}\label{bourgain-lower-bound}
    \epsilon(n) \geq c(d)n^{-\frac{1}{2}-\frac{3}{2d}}.
\end{equation}
When $d=2,3$, this bound was matched up to logarithmic factors in \cite{bourgain1988distribution}, specifically it was shown that for general zonoids $Z$, we have
\begin{equation}
    \epsilon(n) \leq C(d)\begin{cases}
        n^{-\frac{1}{2}-\frac{3}{2d}} \sqrt{\log(n)} & d = 2\\
        n^{-\frac{1}{2}-\frac{3}{2d}} \log(n)^{3/2} & d = 3.
    \end{cases}
\end{equation}
For larger values of $d$ the result in \cite{bourgain1988distribution} gives the worse upper bound of
\begin{equation}
    \epsilon(n) \leq C(d)n^{-\frac{1}{2}-\frac{1}{d-1}} \sqrt{\log(n)}.
\end{equation}
In \cite{bourgain1993approximating} (see also \cite{linhart1989approximation}) it was shown that these bounds can be attained using summands of equal length if $Z = B^{d+1}$. 

The picture was nearly completed in \cite{matouvsek1996improved}, where it was shown that we have
\begin{equation}\label{matousek-result}
    \epsilon(n) \leq C(d)\begin{cases}
        n^{-\frac{1}{2}-\frac{3}{2d}} \sqrt{\log(n)} & d = 2,3\\
        n^{-\frac{1}{2}-\frac{3}{2d}} & d \geq 4.
    \end{cases}
\end{equation}
Moreover, it was shown that the upper bound when $d\geq 4$ can be achieved using summands of equal length for all zonoids $Z$.

In this work, we remove the logarithmic factors in \eqref{matousek-result} when $d=2,3$, i.e., we prove that
\begin{equation}\label{our-result-zonotope}
    \epsilon(n) \leq C(d) n^{-\frac{1}{2}-\frac{3}{2d}}
\end{equation}
for all $d$, and thus provide upper bounds exactly matching (up to a constant factor) the lower bound \eqref{bourgain-lower-bound}. To formulate these results, we pass to the dual setting (see \cite{bourgain1989approximation,matouvsek1996improved}). A symmetric convex body $Z$ is a zonoid iff
\begin{equation}
    \|x\|_{Z^*} := \sup_{z\in Z}x\cdot z = \int_{S^d} |x\cdot y|d\tau(y)
\end{equation}
for a positive measure $\tau$ on $S^d$. The body $Z$ is a zonotope with $n$-summands iff $\tau$ is supported on $n$ points. Since our error measure is scale invariant, we may assume that $\tau$ is a probability distribution. Given these considerations, the bound \eqref{our-result-zonotope} follows from the following result.

\begin{theorem}\label{optimal-zonoid-approximation-theorem}
    There exists a constant $C = C(d)$ such that for any probability measure $\tau$ on the sphere $S^d$, there exists a probability measure $\tau'$ on $S^d$ which is supported on $n$ points, such that
    \begin{equation}
        \sup_{x\in S^d} \left|\int_{S^d}|x\cdot y|d\tau(y) - \int_{S^d}|x\cdot y|d\tau'(y)\right| \leq Cn^{-\frac{1}{2}-\frac{3}{2d}}.
    \end{equation}
\end{theorem}
We remark that our method produces summands of unequal length (i.e. a non-uniform distribution $\tau'$) and we do not know whether this approximation can be achieved using summands of equal length (even for the ball $B^{d+1}$) when $d < 4$.

Recently, there has been renewed interest in the zonoid approximation problem due to its connection with approximation by shallow ReLU$^k$ neural networks \cite{bach2017breaking}. The ReLU$^k$ activation function (simply called ReLU when $k=1$) is defined by
\begin{equation}
    \sigma_k(x) = x_+^k :=  \begin{cases}
        x^k & x \geq 0\\
        0 & x < 0,
    \end{cases}
\end{equation}
where in the case $k=0$ we interpret $0^0=1$ (so that $\sigma_0$ is the Heaviside function). A shallow ReLU$^k$ neural network is a function on $\mathbb{R}^d$ of the form
\begin{equation}
    f_n(x) = \sum_{i=1}^n a_i\sigma_k(\omega_i\cdot x + b_i) = \sum_{i=1}^n a_i(\omega_i\cdot x + b_i)_+^k,
\end{equation}
where the $a_i\in \mathbb{R}$ are coefficients, the $\omega_i\in \mathbb{R}^d$ are directions, the $b_i\in \mathbb{R}$ are the offsets (or biases) and $n$ (the number of terms) is called the width of the network.

Shallow neural networks can be viewed as a special case of non-linear dictionary approximation. Given a Banach space $X$, let $\mathbb{D}\subset X$ be a bounded subset, i.e., $\|\mathbb{D}\| := \sup_{d\in \mathbb{D}} \|d\|_X < \infty$, which we call a dictionary. Non-linear dictionary approximation methods seek to approximate a target function $f$ by elements of the set
\begin{equation}\label{sigma-n-definition-equation}
    \Sigma_n(\mathbb{D}) := \left\{\sum_{i=1}^n a_id_i,~a_i\in \mathbb{R},~d_i\in \mathbb{D}\right\},
\end{equation}
i.e., the set of $n$-term linear combinations of dictionary elements. Note that because the elements $d_i$ are not fixed, this a non-linear set of functions. It is often important to obtain some control on the coefficients $a_i$ in a non-linear dictionary expansion. For this reason, we introduce the set
\begin{equation}
    \Sigma^M_n(\mathbb{D}) := \left\{\sum_{i=1}^n a_id_i,~a_i\in \mathbb{R},~d_i\in \mathbb{D},~\sum_{i=1}^n |a_i| \leq M\right\}
\end{equation}
of non-linear dictionary expansions with $\ell^1$-bounded coefficients. 

We consider using shallow ReLU$^k$ neural networks to approximate functions and derivatives of order up to $k$ uniformly on a bounded domain $\Omega\subset \mathbb{R}^d$, so we take our Banach space $X$ to be the Sobolev space $W^k(L_\infty(\Omega))$ (see for instance \cite{evans2010partial}), with norm given by
\begin{equation}
    \|f\|_{W^{k}(L_\infty(\Omega))} := \sup_{|\alpha| \leq k} \|D^\alpha f\|_{L^\infty(\Omega)}.
\end{equation}
In our analysis, we will often take $\Omega = B^d := \{x:~|x| \leq 1\}$ to be the unit ball in $\mathbb{R}^d$ to simplify the presentation. We remark that our results can be extended to a general bounded domain in a straightforward manner.

Shallow neural networks correspond to non-linear approximation with the dictionary
\begin{equation}\label{Pkd-dictionary-definition}
    \mathbb{D} = \mathbb{P}_k^d := \{\sigma_k(\omega_i\cdot x + b_i),~\omega_i\in S^{d-1},~b_i\in [a,b]\} \subset W^{k}(L_\infty(\Omega)),
\end{equation}
where by positive homogeneity we can take $\omega_i$ on the unit sphere, and the biases $b_i$ are restricted to an interval depending upon $\Omega$ to ensure both the boundedness and expressiveness of $\mathbb{P}_k^d$ (see for instance \cite{siegel2023characterization}). When $\Omega = B^d$ is the unit ball in $\mathbb{R}^d$, we can take $[a,b] = [-1,1]$ for example, since $$\sigma_k(\omega\cdot x + b) + (-1)^k\sigma_k(-\omega\cdot x - b)$$
for $\omega\in S^{d-1}$ and $b\in [-1,1]$ spans the space of polynomials of degree at most $k$ on $B^d$.

A typical class of functions considered in the context of non-linear dictionary approximation is the variation space of the dictionary $\mathbb{D}$, defined as follows. Let
\begin{equation}\label{B-1-d-definition}
    B_1(\mathbb{D}) := \overline{\bigcup_{n=1}^\infty \Sigma^1_n(\mathbb{D})}
\end{equation}
denote the closed symmetric convex hull of the dictionary $\mathbb{D}$ and define the variation norm of a function $f\in X$ by
\begin{equation}
    \|f\|_{\mathcal{K}_1(\mathbb{D})} := \inf\{s > 0:~f\in sB_1(\mathbb{D})\}.
\end{equation}
This construction is also called the gauge norm of the set $B_1(\mathbb{D})$ (see for instance \cite{rockafellar1997convex}), and has the property that the unit ball of the $\mathcal{K}_1(\mathbb{D})$-norm is exactly the closed convex hull $B_1(\mathbb{D})$. We also write
\begin{equation}\label{K-1-d-definition-equation}
    \mathcal{K}_1(\mathbb{D}) := \{f\in X:~\|f\|_{\mathcal{K}_1(\mathbb{D})} < \infty\}
\end{equation}
for the space of functions with finite $\mathcal{K}_1(\mathbb{D})$-norm, which is a Banach space with the $\mathcal{K}_1(\mathbb{D})$-norm (see \cite{siegel2023characterization}, Lemma 1).
The variation norm and variation space have been introduced in different forms in the literature and play an important role in statistics, signal processing, non-linear approximation, and the theory of shallow neural networks (see for instance \cite{barron2008approximation,temlyakov2008greedy,devore1996some,devore1998nonlinear,barron1993universal,jones1992simple,siegel2020approximation,siegel2022high,ongie2019function,parhi2021banach}).

In the case corresponding to shallow ReLU$^k$ neural networks, i.e., $\mathbb{D} = \mathbb{P}_k^d$, the variation space can equivalently be defined via integral representations, which were studied for example in \cite{bach2017breaking,ma2022barron}. Specifically, we have $f\in \mathcal{K}_1(\mathbb{P}_k^d)$ iff there exists a (signed measure) $d\mu$ of bounded variation on $S^{d-1}\times [a,b]$ such that
\begin{equation}
    f(x) = \int_{S^{d-1}\times [a,b]} \sigma_k(\omega\cdot x + b)d\mu(\omega,b)
\end{equation}
pointwise almost everywhere on $\Omega$. Moreover, the variation norm is given by
\begin{equation}
    \|f\|_{\mathcal{K}_1(\mathbb{P}_k^d)} = \inf\left\{\int_{S^{d-1}\times [a,b]}d|\mu|(\omega,b):~f(x) = \int_{S^{d-1}\times [a,b]} \sigma_k(\omega\cdot x + b)d\mu(\omega,b)\right\},
\end{equation}
where the infimum above is taken over all measures with finite total variation giving such a representation of $f$. This is due to the fact that
\begin{equation}\label{eq-221}
    B_1(\mathbb{P}_k^d) = \left\{\int_{S^{d-1}\times [a,b]} \sigma_k(\omega\cdot x + b)d\mu(\omega,b):~\int_{S^{d-1}\times [a,b]}d|\mu|(\omega,b) \leq 1\right\},
\end{equation}
which follows from Lemma 3 in \cite{siegel2023characterization}, and an `empirical' discretization of the integral in \eqref{eq-221} using the fact that half-spaces have bounded VC-dimension \cite{vapnik1971uniform,vapnik1999nature} (note that the closure in \eqref{B-1-d-definition} is taken in the $X = W^k(L_\infty)$-norm). For the details of this argument, we refer to Proposition 2.2 in \cite{yang2023optimal}. In fact, it follows easily from this that we get the same set $B_1(\mathbb{P}_k^d)$, and thus the same variation space $\mathcal{K}_1(\mathbb{P}_k^d)$, even if we take the closure in \eqref{B-1-d-definition} with respect to a weaker norm such as $L_2$.

One important question is how efficiently functions in the variation space $\mathcal{K}_1(\mathbb{D})$ (see \eqref{K-1-d-definition-equation}) can be approximated by non-linear dictionary expansions $\Sigma_n(\mathbb{D})$ with $n$ terms (see \eqref{sigma-n-definition-equation}). When the space $X$ is a Hilbert space (or more generally a type-$2$ Banach space), Maurey's empirical method gives the bound (see for instance \cite{barron1993universal,jones1992simple,pisier1981remarques})
\begin{equation}\label{maurey-jones-barron-rate}
    \inf_{f_n\in \Sigma_n(\mathbb{D})} \|f - f_n\|_X \leq C\|f\|_{\mathcal{K}_1(\mathbb{D})}n^{-\frac{1}{2}}.
\end{equation}
The constant here depends only upon the norm of the dictionary $\|\mathbb{D}\|$ and the type-$2$ constant of the space $X$. Moreover, the norm of the coefficients $a_i$ can be controlled, so that if $f$ is in $B_1(\mathbb{D})$ (the unit ball of $\mathcal{K}_1(\mathbb{D})$), then $f_n$ can be taken in $\Sigma^1_n(\mathbb{D})$. This fact was first applied to neural network approximation by Jones and Barron \cite{barron1993universal,jones1992simple}, and forms the basis of the dimension independent approximation rates obtained for shallow neural networks.

For many dictionaries, for example the dictionaries $\mathbb{P}_k^d$ corresponding to shallow neural networks defined in \eqref{Pkd-dictionary-definition}, the rate \eqref{maurey-jones-barron-rate} can be significantly improved (see for instance \cite{bach2017breaking,klusowski2018approximation,siegel2022sharp,makovoz1996random}). For instance, in the $L_2$-norm we get the rate
\begin{equation}\label{l2-approximation-rate-equation}
    \inf_{f_n\in \Sigma_n(\mathbb{P}_k^d)} \|f - f_n\|_{L_2(\Omega)} \leq C\|f\|_{\mathcal{K}_1(\mathbb{P}_k^d)}n^{-\frac{1}{2}-\frac{2k+1}{2d}},
\end{equation}
and this rate is optimal up to logarithmic factors. With mild control on the coefficients $a_i$ (for instance $|a_i| \leq C$ for a constant $C$) this was proved using an estimate on the metric entropy of $B_1(\mathbb{P}_k^d)$ in \cite{siegel2022sharp}, Theorem 10. In the general case, with no restriction on the coefficients $a_i$, this lower bound follows from VC-dimension arguments. Specifically, using the Sobolev space embedding $W^s(L_2(\Omega))\subset \mathcal{K}_1(\mathbb{P}_k^d)$ for $s = (d+2k+1)/2$ given in \cite{mao2024approximation}, we see that
\begin{equation}
    \sup_{f\in B_1(\mathbb{P}_k^d)}\inf_{f_n\in \Sigma_n(\mathbb{P}_k^d)} \|f - f_n\|_{L_2(\Omega)} \geq C\sup_{\|f\|_{W^s(L_2(\Omega))}\leq 1}\inf_{f_n\in \Sigma_n(\mathbb{P}_k^d)} \|f - f_n\|_{L_2(\Omega)},
\end{equation}
for a constant $C$ depending upon $k,d$ and $\Omega$. The VC-dimension of the function class $\Sigma_n(\mathbb{P}_k^d)$ is bounded by $Cn\log(n)$ (see \cite{bartlett2019nearly} or Lemma 3.11 in \cite{devore2021neural}), so that Theorem 22 in \cite{siegel2023optimal} implies
\begin{equation}\label{approximation-lower-bound-equation}
    \sup_{f\in B_1(\mathbb{P}_k^d)}\inf_{f_n\in \Sigma_n(\mathbb{P}_k^d)} \|f - f_n\|_{L_2(\Omega)} \geq C(n\log n)^{-\frac{1}{2}-\frac{2k+1}{2d}}.
\end{equation}
Thus, the approximation rate given in \eqref{l2-approximation-rate-equation} is tight up to logarithmic factors. It is an intriguing open problem whether these logarithmic factors can be removed.

In this work, we consider approximation rates for the dictionary $\mathbb{P}_k^d$ (see \eqref{Pkd-dictionary-definition}) on its variation space $\mathcal{K}_1(\mathbb{P}_k^d)$ in the $W^{m}(L_\infty)$-norm for $m=0,...,k$, i.e. we consider uniform approximation of both $f$ and its derivatives up to order $m$. This is a much stronger error norm than the $L_2$-norm, and approximating derivatives is important for applications of shallow neural networks to scientific computing (see for instance \cite{siegel2023greedy,lu2022priori,xu2020finite}). For an arbitrary (bounded) dictionary $\mathbb{D}\subset W^{m}(L_\infty)$, no rate of approximation for $\mathcal{K}_1(\mathbb{D})$ by $\Sigma_n(\mathbb{D})$ can be obtained in general, i.e., there exist dictionaries for which the rate can be arbitrarily slow \cite{donahue1997rates}. Thus, our results rely upon the special structure of the $\mathbb{P}_k^d$ dictionary of ReLU$^k$ atoms.

This problem has previously been considered in the case $m=0$, i.e. in the $L_\infty$-norm (see for instance \cite{bach2017breaking,klusowski2018approximation,ma2022uniform,barron1992neural,cheang2000better,yukich1995sup}). In this case, when $k=0$ an approximation rate of
\begin{equation}
    \inf_{f_n\in \Sigma_n(\mathbb{P}_0^d)} \|f - f_n\|_{L_\infty(\Omega)} \leq C\|f\|_{\mathcal{K}_1(\mathbb{P}_0^d)}n^{-\frac{1}{2}-\frac{1}{2d}},
\end{equation}
was proved in \cite{ma2022uniform} using results from geometric discrepancy theory \cite{matouvsek1995tight}. For $k=1$, the aforementioned results on approximating zonoids by zonotopes \cite{matouvsek1996improved} were used in \cite{bach2017breaking} to get a rate of 
\begin{equation}
        \inf_{f_n\in \Sigma_n(\mathbb{P}_1^d)} \|f - f_n\|_{L_\infty(S^d)} \leq C\|f\|_{\mathcal{K}_1(\mathbb{P}_1^d)}\begin{cases}
            n^{-\frac{1}{2}-\frac{3}{2d}}\sqrt{\log{n}} & d = 2,3\\
            n^{-\frac{1}{2}-\frac{3}{2d}} & d \geq 4
        \end{cases}
\end{equation}
on the sphere $S^d$. Finally, when $k\geq 2$, the best known result is \cite{klusowski2018approximation}
\begin{equation}
    \inf_{f_n\in \Sigma_n(\mathbb{P}_k^d)} \|f - f_n\|_{L_\infty(\Omega)} \leq C\|f\|_{\mathcal{K}_1(\mathbb{P}_k^d)} n^{-\frac{1}{2}-\frac{1}{d}} \sqrt{\log{n}}.
\end{equation}
We remark that for all of these results, the coefficients of $f_n$ can be controlled. Specifically, if $f\in B_1(\mathbb{P}_k^d)$, then $f_n$ can be taken in $\Sigma_n^1(\mathbb{P}_k^d)$.

By refining the techniques of discrepancy theory used to obtain these $L_\infty$ bounds, we are able to prove the following result, which is essentially a generalization of Theorem \ref{optimal-zonoid-approximation-theorem}.
\begin{theorem}\label{shallow-network-approximation-dist-result}
    Let $k \geq 0$. For any probability distribution $\tau$ on $S^{d-1}\times [-1,1]$ there exists a probability distribution $\tau'$ supported on at most $n$ points such that for any multi-index $\alpha$ with $|\alpha| \leq k$ we have
    \begin{equation}
        \sup_{x\in B^d} \left|D_x^\alpha\left(\int_{S^d}\sigma_k(\omega\cdot x + b)d\tau(\omega,b) - \int_{S^d}\sigma_k(\omega\cdot x + b)d\tau'(\omega,b)\right)\right| \leq Cn^{-\frac{1}{2}-\frac{2(k - |\alpha|) + 1}{2d}},
    \end{equation}
    where $D_x^\alpha$ denotes the $\alpha$-th order derivative with respect to $x$. Here the constant $C = C(d,k)$ depends only upon $d$ and $k$. 
\end{theorem}
As a corollary, we obtain the following approximation rate, which is sharp up to logarithmic factors by \eqref{approximation-lower-bound-equation}.
\begin{theorem}\label{shallow-network-approximation-theorem}
    Let $0\leq m\leq k$ and $\Omega = B^d$. Then we have the bound
    \begin{equation}
        \inf_{f_n\in \Sigma_n(\mathbb{P}_k^d)} \|f - f_n\|_{W^{m}(L_\infty(\Omega))} \leq C\|f\|_{\mathcal{K}_1(\mathbb{P}_k^d)}n^{-\frac{1}{2}-\frac{2(k-m) + 1}{2d}},
    \end{equation}
    where $C = C(d,k)$ is a constant. 
    Moreover, the coefficients of $f_n$ can be controlled, so if $f\in B_1(\mathbb{P}_k^d)$, then $f_n$ can be taken in $\Sigma_n^1(\mathbb{P}_k^d)$.
\end{theorem}
 We remark that by scaling, Theorems \ref{shallow-network-approximation-dist-result} and \ref{shallow-network-approximation-theorem} can be extended to any bounded domain $\Omega\subset \mathbb{R}^d$. Theorem \ref{shallow-network-approximation-theorem} extends the approximation rates derived in \cite{siegel2022sharp} from the $L_2$-norm to the $L_\infty$-norm, which significantly improves upon existing results \cite{ma2022uniform,bach2017breaking,klusowski2018approximation} when $k \geq 1$. In addition, we obtain approximation rates in the $W^m(L_\infty)$-norm, which enables derivatives and function values to be uniformly approximated simultaneously. The approximation rates given in Theorem \ref{shallow-network-approximation-theorem} are an important building block in obtaining approximation rates for shallow ReLU$^k$ networks on Sobolev and Besov spaces \cite{yang2023optimal,mao2024approximation}.

Theorems \ref{optimal-zonoid-approximation-theorem} and \ref{shallow-network-approximation-dist-result} are proved using a modification of the geometric discrepancy argument used in \cite{matouvsek1996improved}, while Theorem \ref{shallow-network-approximation-theorem} is an easy corollary of Theorem \ref{shallow-network-approximation-dist-result}. Theorem \ref{optimal-zonoid-approximation-theorem} follows essentially as a special case of Theorem \ref{shallow-network-approximation-dist-result} when $k=1$, except that the ReLU activation function is replaced by the absolute value function and the setting is changed to the sphere. For this reason, we only give the complete proof of Theorem \ref{shallow-network-approximation-dist-result}. The changes necessary to obtain Theorem \ref{optimal-zonoid-approximation-theorem} are relatively straightforward and left to the reader. We begin by collecting the necessary geometric and combinatorial facts in Section \ref{lemmas-section}. The proofs of Theorems \ref{shallow-network-approximation-dist-result} and \ref{shallow-network-approximation-theorem} are then given in Section \ref{shallow-relu-network-approximation}.

\section{Geometric Lemmas}\label{lemmas-section}
In this section, we collect and prove some geometric lemmas which are crucial to the proofs of Theorems \ref{optimal-zonoid-approximation-theorem} and \ref{shallow-network-approximation-dist-result}. In particular, in the proof of Theorem \ref{shallow-network-approximation-dist-result} we will need the following covering result.
\begin{lemma}\label{covering-set-ball-lemma}
    Let $P$ be an $N$ point subset of $S^{d-1}\times [-1,1]$ (i.e. a set of $N$ halfspaces) and $0 < \delta < 1$ be given. Then there exists a subset $\mathcal{N}\subset B^d$ of the unit ball (depending upon both $P$ and $\delta$) with $|\mathcal{N}| \leq (C/\delta)^{d}$ such that for any $x\in B^d$ there exists a $z\in \mathcal{N}$ with
    \begin{itemize}
        \item $|x - z| \leq C\delta\sqrt{d}.$
        \item $|P\cap \{(\omega,b)\in S^{d-1}\times [-1,1],~\sign{(\omega\cdot x + b)} \neq \sign{(\omega\cdot z + b)}\}| \leq \delta N.$
    \end{itemize}
    Here $C$ is an absolute constant.
\end{lemma}
A version of this lemma on the sphere, which is required for the proof of Theorem \ref{optimal-zonoid-approximation-theorem}, was proved by Matousek \cite{matouvsek1996improved}.
\begin{lemma}[Lemma 6 in \cite{matouvsek1996improved}]\label{covering-set-lemma}
    Let $P$ be an $N$-point subset of $S^d$ and $0 < \delta < 1$ be given. There exists a subset $\mathcal{N} \subset S^d$ (depending upon both $P$ and $\delta$) with $|\mathcal{N}| \leq C\delta^{-d}$ such that for any $x\in S^d$, there exists a $z\in \mathcal{N}$ with
    \begin{itemize}
        \item $|x - z| \leq \delta.$
        \item $|P\cap \{y\in S^d,~\sign{(y\cdot x)} \neq \sign{(y\cdot z)}\}| \leq \delta N.$
    \end{itemize}
    Here $C = C(d)$ is a constant independent of $N, P$ and $\delta$.
\end{lemma}
The proof of Lemma \ref{covering-set-ball-lemma} follows essentially the same ideas as the proof of Lemma \ref{covering-set-lemma} in \cite{matouvsek1996improved}. However, for completeness we give the proof here as well. The version given in Lemma \ref{covering-set-ball-lemma} explicitly tracks the dimension dependence of the constants and could be used to track the dimension dependence of the constants in Theorems \ref{optimal-zonoid-approximation-theorem} and \ref{shallow-network-approximation-dist-result}.

We begin by recalling the relevant combinatorial background (see for instance \cite{matousek1999geometric}, Chapter 5).

\begin{definition}
    A set system $(X,\mathcal{S})$ consists of a set $X$ and a collection of subsets $\mathcal{S}\subset 2^X$ of $X$.
\end{definition}
The particular set system which we will consider in the proof of Lemma \ref{covering-set-ball-lemma} is given by
\begin{equation}\label{ball-set-system}
    X = S^{d-1}\times [-1,1],~\mathcal{S} = \left\{\{(\omega,b):~\omega\cdot x + b \geq 0\},~x\in B^d\right\}.
\end{equation}
In other words, the elements are halfspaces and the sets consists of all halfspaces containing a given point $x$ in the unit ball.

Given a subset $Y\subset X$, we write $\mathcal{S}|_Y = \{Y\cap S,~S\in \mathcal{S}\}$ for the system $\mathcal{S}$ restricted to the set $Y$.
\begin{definition}[VC-dimension \cite{vapnik1971uniform}]\label{VC-dimension-definition}
    A subset $Y\subset X$ is shattered by $\mathcal{S}$ if $\mathcal{S}|_Y = 2^Y$. The VC-dimension of the set system $(X,\mathcal{S})$ is the cardinality of the largest finite subset of $X$ which is shattered by $\mathcal{S}$ (or infinite if arbitrarily large finite subsets can be shattered).
\end{definition}
An important ingredient in the proof is the following bound on the VC-dimension of the set system given in \eqref{ball-set-system}.
\begin{lemma}[Lemma 3 in \cite{ma2022uniform}]\label{halfspace-VC-bound-lemma}
    The set system in \eqref{ball-set-system} has VC-dimension bounded by $d$.
\end{lemma}
Finally, we will need the following packing lemma for set systems with bounded VC-dimension.
\begin{lemma}[Corollary 1 in \cite{haussler1995sphere}]\label{packing-VC-dimension-lemma}
    Let $(X,\mathcal{S})$ be a set system and $\mu$ a probability measure on the set $X$. Define a distance $d_\mu$ on the collection of subsets $\mathcal{S}$ by
    \begin{equation}
        d_{\mu}(S_1,S_2) = \mu(S_1\Delta S_2).
    \end{equation}
    In other words, $d_\mu$ is the probability that a randomly chosen element from the measure $\mu$ will be in one set but not the other.

    Let $\epsilon > 0$ and suppose that $S_1,...,S_N\in \mathcal{S}$ are such that $d_{\mu}(S_i,S_j) \geq \epsilon$ for all $i\neq j$. Then, if $(X,\mathcal{S})$ has VC-dimension at most $d$, we have
    \begin{equation}
        N \leq \left(\frac{C}{\epsilon}\right)^d
    \end{equation}
    for an absolute constant $C$ (we can take for instance $C = 50$).
\end{lemma}
This was proved by Haussler in the case where $X$ is a finite set and $\mu$ is the counting measure \cite{haussler1995sphere}, which improves upon a weaker result (losing a logarithmic factor) obtained earlier by Dudley \cite{dudley1978central}. The generalization to arbitrary probability measures $\mu$ follows easily from this case and is given as a corollary in \cite{haussler1995sphere}. Similar results have also been noted in the case of certain geometric set systems in \cite{matousek1991cutting}. Using Lemma \ref{packing-VC-dimension-lemma}, we conclude this section with the proof of Lemma \ref{covering-set-ball-lemma}.
\begin{proof}[Proof of Lemma \ref{covering-set-ball-lemma}]
    Consider the set system given in \eqref{ball-set-system} and the probability measure $\mu$ defined by
    \begin{equation}\label{definition-of-mu}
        \mu = \frac{1}{2}\pi + \frac{1}{2}\pi_P,
    \end{equation}
    where $\pi$ is the uniform probability measure on $S^{d-1}\times [-1,1]$, and $\pi_P$ is the empirical measure associated to the set of halfspaces $P$, i.e.,
    \begin{equation}\label{definition-of-pi-P}
        \pi_P = \frac{1}{|P|}\sum_{\theta\in P} \delta_{\theta}
    \end{equation}
    where $\delta_{\theta}$ denotes the Dirac measure at the point $\theta := (\omega,b)\in S^{d-1}\times [-1,1]$. 
    
    Let $x_1,...,x_N\in B^d$ (viewed as elements of the set system $\mathcal{S}$) be a maximal set of points such that $d_{\mu}(x_i,x_j) \geq \delta / 2$. By Lemma \ref{packing-VC-dimension-lemma} and the VC-dimension bound in Lemma \ref{halfspace-VC-bound-lemma}, we have that
    \begin{equation}
        N \leq \left(\frac{2C}{\delta}\right)^d.
    \end{equation}
    Moreover, given an $z\in B^d$ there is an $x_i$ such that $d_\mu(x_i,z) < \delta / 2$ by the maximality of the set $x_1,...,x_N$. From the definition of $\mu$, this means that $d_\pi(x_i,z) < \delta$ and $d_{\pi_P}(x_i,z) < \delta$. Given the form \eqref{definition-of-pi-P} of $\pi_P$, $d_{\pi_P}(x_i,z) < \delta$ is equivalent to
    $$|P\cap \{(\omega,b)\in S^{d-1}\times [-1,1],~\sign{(\omega\cdot x_i + b)} \neq \sign{(\omega\cdot z + b)}\}| < \delta |P| = \delta N.$$
    On the other hand, $d_\mu(x_i,z) < \delta$ implies that
    \begin{equation}\label{expectation-bound-362}
        \mathbb{E}_{\omega\in S^{d-1}} \left[\mathbb{P}(x_i\cdot\omega < b <  z\cdot\omega~\text{or}~z\cdot\omega < b <  x_i\cdot\omega)\right] = \frac{1}{2}\mathbb{E}_{\omega\in S^{d-1}}|(x_i - z)\cdot \omega| < \delta,
    \end{equation}
    where the expectation denotes an average over the sphere $S^{d-1}$, and the probability is over a uniformly random $b\in [-1,1]$. It is well-known that for any fixed unit vector $w\in S^{d-1}$ we have
    \begin{equation}
        \mathbb{E}_{\omega\in S^{d-1}}|w\cdot \omega| \geq cd^{-1/2}
    \end{equation}
    for an absolute constant $c$. Together with \eqref{expectation-bound-362}, this implies that $|x_i - z| < C\delta\sqrt{d}$ as desired.
\end{proof}

\section{Approximation by Shallow ReLU$^k$ Neural Networks}\label{shallow-relu-network-approximation}
In this section, we give the proof of Theorem \ref{shallow-network-approximation-dist-result}.
The key to this proof is the following proposition. We remark that throughout this section, we will use $C$ to denote an unspecified constant which may change from line to line. We will indicate the dependence of this constant whenever it must be clarified.
\begin{proposition}\label{term-decrease-by-a-factor-reluk-proposition}
    Fix an integer $k \geq 0$ and set $c = 1/48$. For any probability distribution $\tau$ on $S^{d-1}\times [-1,1]$ which is supported on $N\geq 6$ points, there exists a probability distribution $\tau'$ supported on at most $(1-c)N$ points, such that for all multi-indices $\alpha$ with $|\alpha|\leq k$, we have
    \begin{equation}\label{proposition-1-bound-equation}
        \sup_{x\in B^d} \left|D_x^\alpha\left(\int_{S^{d-1}\times [-1,1]}\sigma_k(\omega\cdot x + b)d\tau(\omega,b) - \int_{S^{d-1}\times [-1,1]}\sigma_k(\omega\cdot x + b)d\tau'(\omega,b)\right)\right| \leq CN^{-\frac{1}{2}-\frac{2(k - |\alpha|) + 1}{2d}}.
    \end{equation}
    Here $D_x^\alpha$ denotes the $\alpha$-th order derivative with respect to $x$ and $C = C(d,k)$ is a constant depending only upon $d$ and $k$.
\end{proposition}
We remark that the precise value of $c$ is not important to us, it is just important that $c\in (0,1)$ is a fixed constant. Our proof gives the specific value $c = 1/48$, although we have made no attempt to optimize this.

\begin{proof}[Proof of Theorem \ref{shallow-network-approximation-dist-result}]
    The first step is to approximate the probability distribution $\tau$ by a probability distribution with finite (but a priori unbounded) support. This follows the same argument used to prove Proposition 2.2 in \cite{yang2023optimal} and sketched in the remarks following \eqref{eq-221}. We consider drawing $N$ samples $(\omega_1,b_1),...,(\omega_N,b_N)$ independently from the distribution $\tau$ and forming the empirical average 
    \begin{equation}
        \frac{1}{N}\sum_{i=1}^N \sigma_k(\omega_i\cdot x + b_i).
    \end{equation}
    Using the fact that half-spaces have bounded VC-dimension, a standard uniform law of large numbers (see for instance \cite{vapnik1971uniform}) implies that
    \begin{equation}
        \lim_{N\rightarrow \infty} \mathbb{E}\left(\sup_{\substack{x\in B^d\\|\alpha|\leq k}} \left|D_x^\alpha\left(\frac{1}{N}\sum_{i=1}^N \sigma_k(\omega_i\cdot x + b_i) - \int_{S^{d-1}\times [-1,1]}\sigma_k(\omega\cdot x + b)d\tau(\omega,b)\right)\right|\right) = 0,
    \end{equation}
    where the expectation above is taken over the $N$ random samples $(\omega_1,b_1),...,(\omega_N,b_N)$. Thus, for any $\epsilon > 0$, there exists a sufficiently large $N$, and a realization of the samples $(\omega_1,b_1),...,(\omega_N,b_N)$ such that
    \begin{equation}
        \sup_{\substack{x\in B^d\\|\alpha|\leq k}} \left|D_x^\alpha\left(\frac{1}{N}\sum_{i=1}^N \sigma_k(\omega_i\cdot x + b_i) - \int_{S^{d-1}\times [-1,1]}\sigma_k(\omega\cdot x + b)d\tau(\omega,b)\right)\right| < \epsilon.
    \end{equation}
    By choosing $\epsilon$ sufficiently small, it suffices to consider the case where $\tau$ is supported on a finite (though a priori unbounded) number of points.
    
    We repeatedly apply Proposition \ref{term-decrease-by-a-factor-reluk-proposition} to the distribution $\tau_0 := \tau$ to obtain a sequence of distributions $\tau_0, \tau_1,\tau_2,...$. Letting $\|\tau_j\|_0$ denote the size of the support of the (finite) distribution $\tau_j$, Proposition \ref{term-decrease-by-a-factor-reluk-proposition} implies that $\|\tau_{j+1}\|_0 \leq (1-c)\|\tau_j\|_0$, and that for all $|\alpha| \leq k$ we have
    \begin{equation}\label{iterative-error-bound-435}
        \sup_{x\in B^d} \left|D_x^\alpha\left(\int_{S^{d-1}\times [-1,1]}\sigma_k(\omega\cdot x + b)d\tau_j(\omega,b) - \int_{S^{d-1}\times [-1,1]}\sigma_k(\omega\cdot x + b)d\tau_{j+1}(\omega,b)\right)\right| \leq C\|\tau_j\|_0^{-\frac{1}{2}-\frac{2(k - |\alpha|) + 1}{2d}}.
    \end{equation}
    By adjusting the constant, it suffices to consider the case where $n \geq 6$. We set $\tau' = \tau_i$ where $i$ is the smallest index such that $\|\tau_i\|_0 \leq n$. By summing up the error bounds in \eqref{iterative-error-bound-435}, we see that for any $|\alpha| \leq k$ the left-hand side of \eqref{proposition-1-bound-equation} is bounded by
    \begin{equation}
        C\sum_{j=0}^{i-1} \|\tau_j\|_0^{-\frac{1}{2}-\frac{2(k - |\alpha|) + 1}{2d}}.
    \end{equation}
    Since by construction $\|\tau_{i-1}\|_0 > n$ and $\|\tau_{j+1}\|_0 \leq (1-c)\|\tau_j\|_0$ it follows inductively that $\|\tau_j\|_0 > (1-c)^{j-i+1}n$ for $j=0,...,i-1$. This gives the bound
    \begin{equation}
    \begin{split}
        Cn^{-\frac{1}{2}-\frac{2(k-|\alpha|)+1}{2d}}\left(\sum_{j=0}^{i-1} (1-c)^{(i-j-1)\left(\frac{1}{2}+\frac{2(k - |\alpha|) + 1}{2d}\right)}\right) &\leq Cn^{-\frac{1}{2}-\frac{2(k-|\alpha|)+1}{2d}}\left(\sum_{j=0}^{\infty} (1-c)^{j\left(\frac{1}{2}+\frac{2(k - |\alpha|) + 1}{2d}\right)}\right)\\
        &\leq Cn^{-\frac{1}{2}-\frac{2(k-|\alpha|)+1}{2d}},
    \end{split}
    \end{equation}
    as desired.
\end{proof}
\begin{proof}[Proof of Theorem \ref{shallow-network-approximation-theorem}]
    Suppose without loss of generality that $\|f\|_{\mathcal{K}_1(\mathbb{P}_k^d)} \leq 1$, i.e. that $f\in B_1(\mathbb{P}_k^d)$.
    
    By definition, this means that for any $\epsilon > 0$ there exist parameters $(\omega_1,b_1),...,(\omega_N,b_N)\in S^d\times [-1,1]$ and weights $a_1,...,a_N\in \mathbb{R}$ (for a sufficiently large $N$) such that
    \begin{equation}\label{f-approximation-273}
        \left\|f - \sum_{i=1}^N a_i\sigma_k(\omega_i\cdot x + b_i)\right\|_{W^{k}(L_\infty(S^d))} < \epsilon,
    \end{equation}
    and $\sum_{i=1}^N |a_i| \leq 1$. The next step is to approximate the sum in \eqref{f-approximation-273} by an element in $\Sigma_n^1(\mathbb{P}_k^d)$. To do this, we split the sum into its positive and negative parts, i.e., we write
    \begin{equation}\label{positive-negative-decomposition-277}
        \sum_{i=1}^N a_i\sigma_k(\omega_i\cdot x + b_i) = \sum_{a_i > 0} a_i\sigma_k(\omega_i\cdot x + b_i) - \sum_{a_i < 0} |a_i|\sigma_k(\omega_i\cdot x + b_i).
    \end{equation}
    By considering the positive and negative pieces separately, we essentially reduce to the case where all $a_i$ are positive. In this case, the sum can be written
    \begin{equation}
        \sum_{i=1}^N a_i\sigma_k(\omega_i\cdot x + b_i) = \int_{S^{d-1}\times [-1,1]}\sigma_k(\omega\cdot x + b)d\tau(\omega,b)
    \end{equation}
    for a probablity measure $\tau$ supported on at most $N$ points. 

    Applying Theorem \ref{shallow-network-approximation-dist-result} gives an $f_n\in \Sigma_{n/2}^1(\mathbb{P}_k^d)$ such that
    \begin{equation}
        \left\|f_n - \sum_{i=1}^N a_i\sigma_k(\omega_i\cdot x + b_i)\right\|_{W^{m}(L_\infty(S^d))} \leq Cn^{-\frac{1}{2}-\frac{2(k-m) + 1}{2d}},
    \end{equation}
    whenever $a_i \geq 0$ and $\sum_{i=1}^N a_i = 1$. Applying this to the positive and negative parts in \eqref{positive-negative-decomposition-277} and summing them gives an $f_n\in \Sigma_{n}^1(\mathbb{P}_k^d)$ such that
    \begin{equation}
        \left\|f - f_n\right\|_{W^{m}(L_\infty(S^d))} \leq Cn^{-\frac{1}{2}-\frac{2(k-m) + 1}{2d}} + \epsilon.
    \end{equation}
    Since $\epsilon > 0$ was arbitrary, this completes the proof.
\end{proof}

It remains to prove Proposition \ref{term-decrease-by-a-factor-reluk-proposition}. The proof utilizes the ideas of geometric discrepancy theory and borrows many ideas from the proof of Proposition 9 in \cite{matouvsek1996improved}. However, Proposition 9 in \cite{matouvsek1996improved} only deals with uniform distributions, and a few key modifications are required to deal with the case of `unbalanced' distributions $\tau$, which enables us to remove the logarithmic factors in all dimensions in Theorems \ref{optimal-zonoid-approximation-theorem}, \ref{shallow-network-approximation-dist-result}, and \ref{shallow-network-approximation-theorem}. In addition, dealing with the higher order smoothness of the ReLU$^k$ activation function introduces significant technical difficulties.

    We first introduce some notation. Throughout the proof, we will write $\theta := (\omega,b)\in S^{d-1}\times [-1,1]$ as shorthand for a pair of parameters $\omega$ and $b$.
    
    We will also need to work with tensors in order to handle higher order derivatives of multivariate functions. Our tensors will be defined on the space $\mathbb{R}^{d}$, so let $I := \{1,...,d\}$ denote the relevant indexing set. A tensor $X$ of order $m$ is simply an array of numbers indexed by a tuple $\textbf{i}\in I^m$. Note that vectors in $\mathbb{R}^{d}$ are tensors of order one. We adopt the $\ell^\infty$ norm on the space of degree $m$ tensors, i.e.,
    \begin{equation}
        \|X\| := \max_{\textbf{i}\in I^m} |X_\textbf{i}|.
    \end{equation}  
    
    Given tensors $X$ and $Y$ of orders $m_1$ and $m_2$, their tensor product, which is a tensor of order $m_1 + m_2$, is defined in the standard way by
    \begin{equation}
        (X\otimes Y)_{\textbf{i}\textbf{j}} = X_{\textbf{i}}Y_{\textbf{j}},
    \end{equation}
    where $\textbf{i}\in I^{m_1}$, $\textbf{j}\in I^{m_2}$ and $\textbf{i}\textbf{j}$ denotes concatenation. We will also write $X^{\otimes r}$ for the $r$-fold tensor product of $X$ with itself.
    Supposing that $m_1 \geq m_2$, we define the contraction, which is a tensor of order $m_1 - m_2$ by
    \begin{equation}
        \langle X,Y\rangle_{\textbf{i}} = \sum_{\textbf{j}\in I^{m_2}} X_{\textbf{i}\textbf{j}}Y_{\textbf{j}}.
    \end{equation}
    Note that since we will be exclusively working with $\mathbb{R}^{d}$ with the standard inner product, to simplify the notation and presentation we will not make the distinction between covariance and contravariance in the following. 
    
    We remark that repeated contraction can be written in terms of the tensor product in the following way
    \begin{equation}\label{contraction-tensor-product-identity}
        \langle \langle X,Y\rangle, Z\rangle = \langle X, Y\otimes Z\rangle,
    \end{equation}
    and also note the inequality (for tensors of degree $m \leq k$)
    \begin{equation}\label{contraction-l-infty-inequality}
        \left\|\langle X,Y\rangle\right\| \leq C\|X\|\|Y\|,
    \end{equation}
    where $C = C(d,k) = d^k$.
    
    Given an order $0 \leq m\leq k$, we denote the $m$-th derivative (tensor) of the ReLU$^k$ function (with $\theta := (\omega,b)$ fixed) by
    \begin{equation}\label{sigma-k-m-definition}
        \sigma_k^{(m)}(x;\theta) = D_x^m[\sigma_k(\omega\cdot x + b)] = \begin{cases}
            \frac{k!}{(k-m)!}(\omega\cdot x + b)^{k-m}\omega^{\otimes m} & \omega\cdot x + b \geq 0\\
            0 & \omega\cdot x + b < 0.
        \end{cases}
    \end{equation}
    In order to deal with the degree $k$ smoothness of the activation function we will need to utilize higher-order Taylor polynomials. Given points $x_1,x_2\in S^d$, parameters $\theta\in S^{d-1}\times [-1,1]$, an order $0\leq m\leq k$, and a number of terms $0\leq r\leq k-m$, we denote by
    \begin{equation}\label{taylor-expansion-definition}
        \mathcal{T}^{m,r}_{x_1}(x_2;\theta) := \sum_{q=0}^r \frac{1}{q!}\left\langle\sigma_k^{(m+q)}(x_1;\theta), (x_2 - x_1)^{\otimes q}\right\rangle
    \end{equation}
    the $r$-th order Taylor polynomials of $\sigma_k^{(m)}(x;\theta)$ around $x_1$ evaluated at $x_2$.

\begin{proof}[Proof of Proposition \ref{term-decrease-by-a-factor-reluk-proposition}]
    Note that since $\tau$ is supported on $N$ points 
    the integral which we are trying to approximate in Proposition \ref{term-decrease-by-a-factor-reluk-proposition} is given by (recall here $\theta := (\omega,b)$)
    \begin{equation}\label{sum-representation-146}
        \int_{S^{d-1}\times [-1,1]}\sigma_k(\omega\cdot x + b)d\tau(\omega,b) = \sum_{\theta\in S} a_{\theta}\sigma_k(\omega\cdot x + b),
    \end{equation}
    where $S\subset S^{d-1}\times [-1,1]$ with $|S| = N$, and the coefficients $a_{\theta}$ satisfy $a_{\theta}\geq 0$ and  $\sum_{\theta\in S} a_{\theta} = 1$.
    
    Let $M$ denote the median of the coefficients $a_{\theta}$ and set
    $$
    S_- = \{\theta\in S:~a_{\theta} \leq M\},~~S_+ = \{\theta\in S:a_{\theta} > M\}.
    $$
    This gives a decomposition of the sum in \eqref{sum-representation-146} in terms of its large and small coefficients
    \begin{equation}
        \sum_{\theta\in S} a_{\theta}\sigma_k(\omega\cdot x + b) = \sum_{\theta\in S_-} a_{\theta}\sigma_k(\omega\cdot x + b) + \sum_{\theta\in S_+} a_{\theta}\sigma_k(\omega\cdot x + b).
    \end{equation}
    We will leave the second sum, i.e., the large sum, untouched and approximate the small sum by
    \begin{equation}\label{term-count-decrease-approximation}
        \sum_{\theta\in S_-} a_{\theta}\sigma_k(\omega\cdot x + b) \approx \sum_{\theta\in T} b_{\theta}\sigma_k(\omega\cdot x + b),
    \end{equation}
    where $T\subset S_-$ and the $b_\theta\geq 0$ are new coefficients such that the following conditions are satisfied:
    \begin{enumerate}
        \item $|T| \leq (1-c)|S_-|$ (our argument will give here $c = 1/24$)
        \item $\sum_{\theta\in T}b_{\theta} = \sum_{\theta\in S_-} a_{\theta}$
        \item The error of approximation satisfies (using the $\ell^\infty$-tensor norm we have introduced)
    \begin{equation}\label{error-bound-162}
        \sup_{x\in S^d}\left\|\sum_{\theta\in S_-} a_{\theta}\sigma_k^{(m)}(x;\theta) - \sum_{\theta\in T} b_{\theta}\sigma_k^{(m)}(x;\theta)\right\| \leq CN^{-\frac{1}{2} - \frac{2(k-m)+1}{2d}}
    \end{equation}
    for $m=0,...,k$.
    \end{enumerate}
We will then complete the proof by setting (here $\delta_{\theta}$ denotes the Dirac delta distribution at $\theta$)
    \begin{equation}
        \tau' = \sum_{\theta\in T}b_{\theta} \delta_{\theta} + \sum_{\theta\in S_+} a_{\theta}\delta_{\theta}.
    \end{equation}
    Since $|S_-| \geq N/2$, the first condition above ensures that the support of $\tau'$ is at most $(1-\frac{c}{2})N$, while the second condition ensures that $\tau'$ is still a probability distribution, and the third condition gives the desired error bound.

    We now turn to the heart of the proof, which is constructing an approximation \eqref{term-count-decrease-approximation} satisfying the three conditions given above. Note first that by construction, we have
    \begin{equation}\label{bound-on-coefficients-small-half}
        \max_{\theta\in S_-} a_\theta \leq M \leq \frac{2}{N},
    \end{equation}
    i.e., all of the coefficients in the small half are at most $2/N$. This holds since at least half (i.e. at least $N/2$) of the $a_\theta$ are at least as large as the median $M$ and $\sum_{\theta\in S} a_{\theta} = 1$.

    Next, we construct a multiscale covering of the ball using Lemma \ref{covering-set-ball-lemma}. For $l=1,...,L$ with $2^L > N$ we apply Lemma \ref{covering-set-ball-lemma} with $P = S_-$ and $\delta = 2^{-l}$ to obtain a sequence of sets $N_l\subset B^d$ with $|N_l| \leq (C2^{l})^d$ such that for any $x\in B^d$ there exists a $z\in N_l$ with $|x-z| \leq C2^{-l}\sqrt{d}$ and
    \begin{equation}\label{theta-number-bound-582}
        |\{\theta\in S_-:\sign{(\omega\cdot x + b)} \neq \sign{(\omega\cdot z + b)}\}| \leq 2^{-l} |S_-| \leq 2^{-l}N.
    \end{equation}
    Given a point $x\in S^d$, we denote by $\pi_l(x)$ the point $z\in N_l$ satisfying these properties (if this point is not unique we choose one arbitrarily for each $x$). 

    For each level $l=1,...,L$, each point $x\in N_l$, and each index $m=0,..,k$ we consider the function
    \begin{equation}\label{definition-of-phi-m}
        \phi_{x,l}^m(\theta) = \begin{cases}
            \sigma_k^{(m)}(x;\theta) - \mathcal{T}^{m,k-m}_{\pi_{l-1}(x)}(x;\theta) & l \geq 2\\
            \sigma_k^{(m)}(x;\theta) & l=1,
        \end{cases}
    \end{equation}
    where $\mathcal{T}^{m,k-m}_{\pi_{l-1}(x)}(x;\theta)$ is the $(k-m)$-th order Taylor polynomial of $\sigma_k^{(m)}(x;\theta)$ defined in \eqref{taylor-expansion-definition}.

    We note the following bounds on $\phi_{x,l}^m(\theta)$. First, if $\sign{(\omega\cdot x + b)} = \sign{(\omega\cdot \pi_{l-1}(x) + b)}$ (for $l \geq 2$), then
    \begin{equation}\label{phi-bound-zero}
        \phi_{x,l}^m(\theta) = 0.
    \end{equation}
    This holds since on the half space $\{x:\omega\cdot x + b\geq 0\}$ the function $\sigma_k^{(m)}(x;\theta)$ is a polynomial of degree $k-m$ in $x$. Thus on this half-space it is equal to its $(k-m)$-th order Taylor polynomial about any point. So if $x$ and $\pi_{l-1}(x)$ both lie in this half-space, then the difference in \eqref{definition-of-phi-m} vanishes. On the other hand, if $x$ and $\pi_{l-1}(x)$ both lie in the complement, then all terms in \eqref{definition-of-phi-m} are $0$.

    On the other hand, for any $x\in B^d$ and $\theta\in S^{d-1}\times [-1,1]$ we have the bound
    \begin{equation}\label{bound-on-phi-nonzero}
        \|\phi_{x,l}^m(\theta)\| \leq C2^{-l(k-m)},
    \end{equation}
    where $C = C(d,k)$. This holds since $\sigma_k^{(m)}(x;\theta)$ (as a function of $x$) has $(k-m)$-th order (weak) derivatives which are bounded by $C(k) = k!2^{k-m}$ for $x\in B^d$ by \eqref{sigma-k-m-definition}. Thus, using Taylor's theorem the difference in \eqref{definition-of-phi-m} is bounded by
    \begin{equation}
        \left\|\sigma_k^{(m)}(x;\theta) - \mathcal{T}^{m,k-m}_{\pi_{l-1}(x)}(x;\theta)\right\| \leq C|x - \pi_{l-1}(x)|^{k-m} \leq C2^{-l(k-m)},
    \end{equation}
    for $C = C(d,k)$. When $l=1$ we also trivially obtain the bound \eqref{bound-on-phi-nonzero}, since $\sigma_k^{(m)}(x,\theta)$ is bounded for $x\in B^d$ (note that $C$ depends on $k$).
    
    The next step is to decompose the functions $\sigma_k^{(m)}(x;\theta)$ with $\theta\in S_-$ in terms of the $\phi_{x,l}^m(\theta)$. This is captured in the following technical lemma.
    \begin{lemma}\label{relu-k-representation-lemma}
    Let $\phi_{x,l}^m$ be defined by \eqref{definition-of-phi-m}. For $x\in B^d$ define $x_L = \pi_L(x)$ and $x_l = \pi_l(x_{l+1})$ for $1 \leq l < L$. 
    
    Then for any $m=0,...,k$, $x\in B^d$ and $\theta\in S_-$ we have
    \begin{equation}\label{sigma-decomposition-reluk}
        \sigma_k^{(m)}(x;\theta) = \sum_{l=1}^L\phi^{m}_{x_j,j}(\theta) + \sum_{i=1}^{k-m}\sum_{l=1}^L  \left\langle\phi^{m+i}_{x_l,l}(\theta),\Gamma^{m}_{i,l}(x)\right\rangle,
    \end{equation}
    for a collection of tensors $\Gamma^{m}_{i,l}(x)$ depending upon $x$ which satisfy the bound
        \begin{equation}\label{Gamma-x-tensor-bound}
            \|\Gamma^{m}_{i,l}(x)\| \leq C2^{-il},
        \end{equation}
        for a constant $C(d,k)$.
    \end{lemma}
    To give some intuition on the meaning of Lemma \ref{relu-k-representation-lemma}, we remark that when $m=k$ it reduces to the telescoping sum
    \begin{equation}
        \sigma_k^{(k)}(x;\theta) = \sigma_k^{(k)}(x_1;\theta) + \sum_{l=2}^L\left(\sigma_k^{(k)}(x_l;\theta) - \sigma_k^{(k)}(x_{l-1};\theta)\right) = \sigma_k^{(k)}(x_L;\theta).
    \end{equation}
    This equality holds since $x$ and $x_L$ lie on the same side of every hyperplane defined by a $\theta\in S_-$ (because $2^{-L}N < 1$ and so the cardinality on the right-hand side of \eqref{theta-number-bound-582} must be $0$).
    
    In order to incorporate higher order smoothness, we must consider differences between $\sigma_k^{(m)}$ and its Taylor expansion around the points $x_l$, which results in the technical statement of the lemma.
    The proof of Lemma \ref{relu-k-representation-lemma} is a technical, but relatively straightforward calculation, so we postpone it until the end of this section and complete the proof of Proposition \ref{term-decrease-by-a-factor-reluk-proposition} first.

    Using the decomposition \eqref{sigma-decomposition-reluk}, we write the LHS of \eqref{error-bound-162} as
    \begin{equation}
    \begin{split}
        \sup_{x\in S^d}\left\|\sum_{l=1}^L\left[\sum_{\theta\in S_-} a_{\theta}\phi^{m}_{x_l,l}(\theta) - \sum_{\theta\in T} b_{\theta}\phi^{m}_{x_l,l}(\theta)\right]\right.+ \left.\sum_{l=1}^L\sum_{i=1}^{k-m}\left\langle\sum_{\theta\in S_-} a_{\theta}\phi^{m+i}_{x_l,l}(\theta) - \sum_{\theta\in T} b_{\theta}\phi^{m+i}_{x_l,l}(\theta), \Gamma^{m}_{i,l}(x)\right\rangle\right\|.
    \end{split}
    \end{equation}
    Using the triangle inequality, the bound \eqref{contraction-l-infty-inequality}, and the bound \eqref{Gamma-x-tensor-bound}, we see that it suffices to find the subset $T\subset S_-$ with $|T| \leq (1-c)|S_-|$, and new coefficients $b_{\theta}\geq 0$ satisfying $\sum_{\theta\in T}b_{\theta} = \sum_{\theta\in S_-} a_{\theta}$ such that for $m=0,...,k$ we have
    \begin{equation}\label{error-decomposition-241}
        \sum_{l=1}^L \sum_{i=0}^{k-m} 2^{-il}\sup_{x\in N_l}\left\|\sum_{\theta\in S_-} a_{\theta}\phi^{m+i}_{x,l}(\theta) - \sum_{\theta\in T} b_{\theta}\phi^{m+i}_{x,l}(\theta)\right\| \leq CN^{-\frac{1}{2} - \frac{2(k-m) + 1}{2d}}
    \end{equation}
    for a constant $C = C(d,k)$.

    To find this set $T$ and new coefficients $b_{\theta}$, we choose disjoint subsets $P_1,..,P_t\subset S_-$ of size $3$ which together cover at least half of the elements of $S_-$, i.e., such that
    \begin{equation}\label{number-of-P-lower-bound}
        \left|\bigcup_{i=1}^t P_i\right| \geq |S_-|/2.
    \end{equation}
    (Note that here we need $|S_-| \geq 3$ which follows from $N \geq 6$.)
    We denote the three elements of each set $P_j$ by
    $$
        P_j = \{u_j,v_j,w_j\},
    $$
    which are ordered so that the coefficients satisfy $0 \leq a_{u_j}\leq a_{v_j}\leq a_{w_j}$. 

    Based upon the partition $P_1,...,P_t$, we will use a modification of the partial coloring argument given in \cite{matouvsek1996improved} (the idea is originally due to Spencer \cite{spencer1985six} and Beck \cite{beck1984some}). The key difference is in how we use the partial coloring to reduce the number of terms in the sum over $S_-$.

    Given a partial coloring $\chi:\{1,...,t\}\rightarrow \{-1,0,1\}$, we transform the sum $\sum_{\theta\in S_-} a_{\theta} \sigma_k(\omega\cdot x + b)$ in the following way:
    \begin{itemize}
        \item If $\chi(j) = 1$, we remove the term corresponding to $u_j$, double the coefficient $a_{v_j}$ of the term corresponding to $v_j$, and add the difference $a_{u_j} - a_{v_j}$ to the coefficient $a_{w_j}$ of the term corresponding to $w_j$.
        \item If $\chi(j) = -1$, we do the same but reverse the roles of $u_j$ and $v_j$.
        \item If $\chi(j) = 0$, we leave the terms corresponding to $u_j,v_j$ and $w_j$ unchanged.
    \end{itemize}

    This results in a transformed sum $\sum_{\theta\in T} b_{\theta} \sigma_k(\omega\cdot x + b)$ over a set $T\subset S_-$ and with coefficients $b_{\theta}$ for $\theta\in T$ described as follows. Let
    \begin{equation}
        R_j = \begin{cases}
            \emptyset & \chi(j) = 0\\
            \{u_j\} & \chi(j) = 1\\
            \{v_j\} & \chi(j) = -1,
        \end{cases}
    \end{equation}
    denote the removed set for each $P_j$.
    Then the set $T$ is given by
    \begin{equation}
        T = S_-\setminus \left(\bigcup_{j=1}^t R_j\right),
    \end{equation}
    and for $\theta\in T$ the coefficients $b_{\theta}$ are given by
    \begin{equation}
        b_{\theta} = \begin{cases}
            a_{\theta} & \theta\notin \bigcup_{j}P_j\\
            (1+\chi(j))a_{\theta} & \theta = v_j\\
            (1-\chi(j))a_{\theta} & \theta = u_j\\
            a_{\theta} + \chi(j)(a_{u_j} - a_{v_j}) & \theta = w_j.
        \end{cases}
    \end{equation}
    We have constructed this transformation so that $$\sum_{\theta\in T} b_{\theta} = \sum_{\theta\in S_-} a_{\theta},$$
    the $b_\theta \geq 0$ since $w_j$ has the largest coefficient among the elements of $P_j$, and for any $x\in S^d$ the error in the $m$-th derivative is given by
    \begin{equation}
    \begin{split}
        \sum_{\theta\in S_-} a_{\theta}\sigma_k^{(m)}(x;\theta)& - \sum_{\theta\in T} b_{\theta}\sigma_k^{(m)}(x;\theta) =\\
        &\sum_{j=1}^t \chi(j)\left[-a_{u_j}\sigma_k^{(m)}(x;u_j) + a_{v_j}\sigma_k^{(m)}(x;v_j) + (a_{u_j} - a_{v_j})\sigma_k^{(m)}(x;w_j)\right].
    \end{split}
    \end{equation}
    Using the linearity of the derivative and the definition of the Taylor polynomial \eqref{taylor-expansion-definition}, this implies that for any $x\in S^d$, any level $l = 1,...,L$, and any $m=0,...,k$ we have
    \begin{equation}\label{error-formula-306}
        \sum_{\theta\in S_-} a_{\theta} \phi_{x,l}^m(\theta) - \sum_{\theta\in T} b_{\theta}\phi_{x,l}^m(\theta) = \sum_{j=1}^t \chi(j) \Psi^m_{x,l,j},
    \end{equation}
    where we have defined
    \begin{equation}\label{definition-of-psi}
    \begin{split}
        \Psi^m_{x,l,j} = -a_{u_j}\phi^m_{x,l}(u_j) + a_{v_j}\phi^m_{x,l}(v_j) + (a_{u_j} - a_{v_j})\phi^m_{x,l}(w_j).
    \end{split}
    \end{equation}
    Further, for any index $j$ such that $\chi(j) \neq 0$, we have eliminated one term (either $u_j$ or $v_j$) from the sum. Thus
    \begin{equation}\label{T-decrease-equation-317}
        |T| = |S| - |\{j:~\chi(j) \neq 0\}|.
    \end{equation}    
    
    We proceed to find a partial coloring $\chi:\{1,...,t\}\rightarrow \{-1,0,1\}$ with a positive fraction of non-zero entries, i.e., with $|\{j:~\chi(j) \neq 0\}| \geq ct$, such that for $m=0,...,k$
    \begin{equation}\label{partial-coloring-bound}
        \sum_{l=1}^L \sum_{i=0}^{k-m} 2^{-il}\sup_{x\in N_l} \left\|\sum_{j=1}^t\chi(j)\Psi^{m+i}_{x,l,j}\right\| \leq CN^{-\frac{1}{2} - \frac{2(k-m)+1}{2d}},
    \end{equation}
    for a constant $C = C(d,k)$. We remark that ultimately our argument will give $c = 1/4$.

    By \eqref{error-formula-306} this will guarantee that the LHS in \eqref{error-decomposition-241} is sufficiently small, and by \eqref{T-decrease-equation-317} this will guarantee that the set $T$ is small enough, since by \eqref{number-of-P-lower-bound}
    \begin{equation}
        |T| = |S| - |\{j:~\chi(j) \neq 0\}| \leq |S| - ct \leq \left(1 - \frac{c}{6}\right)|S_-|.
    \end{equation}
    The existence of such a partial coloring $\chi$ follows from a well-known technique in discrepancy theory called the partial coloring method.

    Given a (total) coloring $\epsilon:\{1,...,t\}\rightarrow \{\pm 1\}$ we consider the quantities
    \begin{equation}\label{definition-of-E}
        E^m_{x,l}(\epsilon) := \sum_{j=1}^t\epsilon(j)\Psi^m_{x,l,j}
    \end{equation}
    for each $x\in N_l$. We would like to find a coloring $\epsilon$ such that $\|E^m_{x,l}(\epsilon)\| \leq \Delta^m_l$ for all $l=1,...,L$, $m=0,...,k$ and $x\in N_l$, where the $\Delta^m_l$ are suitable parameters chosen so that
    \begin{equation}\label{condition-on-delta-k}
        \sum_{l=1}^L \sum_{i=0}^{k-m} 2^{-il} \Delta^{m+i}_l \leq CN^{-\frac{1}{2}-\frac{2(k-m) + 1}{2d}},
    \end{equation}
    for $m=0,...,k$.
    
    One strategy would be to choose $\epsilon$ uniformly at random, bound the tail of the random variable $E^m_{x,l}(\epsilon)$, and use a union bound over $x\in N_l$. Unfortunately, this strategy will lose a factor $\sqrt{\log{N}}$. The ingenious method to get around this, due to Spencer \cite{spencer1985six} and Beck \cite{beck1984some}, is to show that instead there exist \textit{two} colorings $\epsilon_1$ and $\epsilon_2$ such that $\|E^m_{x,l}(\epsilon_1) - E^m_{x,l}(\epsilon_2)\| \leq \Delta^m_l$ for all $l=1,...,L$, $m=0,...,k$, and $x\in N_l$, and such that $\epsilon_1$ and $\epsilon_2$ differ in many indices, i.e.
    \begin{equation}
        |\{j:~\epsilon_1(j)\neq \epsilon_2(j)\}| \geq ct
    \end{equation}
    for an absolute constant $c$. Then $\chi = \frac{1}{2}(\epsilon_1 - \epsilon_2)$ gives the desired partial coloring.

    We will prove the existence of these two colorings $\epsilon_1$ and $\epsilon_2$ for suitably chosen parameters $\Delta^m_l$ satisfying \eqref{condition-on-delta-k}. To help organize this calculation, it is convenient to introduce the notion of the entropy of a discrete distribution (see for instance \cite{matousek1999geometric,alon2016probabilistic}). (Note that for simplicity all of the logarithms in the following are taken with base $2$.)
\begin{definition}
    Let $X$ be a discrete random variable, i.e. the range of $X$ is a countable set $\Lambda$. The entropy of $X$ is defined by
    \begin{equation}
        H(X) = -\sum_{\lambda\in \Lambda} p_\lambda\log(p_\lambda),
    \end{equation}
    where $p_\lambda = \mathbb{P}(X = \lambda)$ is the probability of the outcome $\lambda$.
\end{definition}
One important property of the entropy we will use is subadditivity, i.e. if $X = (X_1,...,X_r)$, then
\begin{equation}\label{subadditivity-entropy}
    H(X) \leq \sum_{j=1}^r H(X_j),
\end{equation}
where we have equality in the above bound when the components $X_j$ of $X$ are independent.

An important component of the calculation is the following lemma from \cite{matouvsek1996improved} (see also \cite{matouvsek1996discrepancy,alon2016probabilistic}).
\begin{lemma}[Lemma 11 in \cite{matouvsek1996improved}]\label{entropy-partial-coloring-lemma}
    Let $\epsilon:\{1,...,t\}\rightarrow \{\pm 1\}$ be a uniformly random coloring. Let $b$ be a function of $\epsilon$ and suppose that the entropy satisfies $H(b(\epsilon)) \leq t/5$. Then there exist two colorings $\epsilon_1,\epsilon_2$ differing in at least $t/4$ components such that $b(\epsilon_1) = b(\epsilon_2)$.
\end{lemma}
We utilize this lemma in the following way. Take each entry of the (tensor-valued) random variable $E^m_{x,l}(\epsilon)$ defined in \eqref{definition-of-E} and round it to the nearest multiple of the (still undetermined) parameter $\Delta^m_l$. This results in a random variable
\begin{equation}
    b^m_{x,l}(\epsilon) = [(\Delta_l^m)^{-1}E^m_{x,l}(\epsilon)],
\end{equation}
where $[\cdot]$ denote the (component-wise) nearest integer function. Note that if $b^m_{x,l}(\epsilon_1) = b^m_{x,l}(\epsilon_2)$, then it follows that $\|E^m_{x,l}(\epsilon_1) - E^m_{x,l}(\epsilon_2)\| \leq \Delta^m_l$. Applying Lemma \ref{entropy-partial-coloring-lemma} and the subadditivity of the entropy \eqref{subadditivity-entropy}, we see that if
\begin{equation}\label{entropy-sum-bound-354}
    \sum_{l=1}^L\sum_{m=0}^k\sum_{x\in N_l} H(b^m_{x,l}(\epsilon)) \leq t/5
\end{equation}
for an appropriate choice of $\Delta^m_l$ satisfying \eqref{condition-on-delta-k}, then there exist two colorings $\epsilon_1$ and $\epsilon_2$ satisfying the desired condition with $c = 1/4$.

It remains to choose the parameters $\Delta^m_l$ satisfying \eqref{condition-on-delta-k} and to bound the sum in \eqref{entropy-sum-bound-354}. For this, we utilize the following lemma from \cite{matouvsek1996improved}, which bounds the entropy of a `rounded' random variable in terms of the tails of the underyling real-valued random variable.
\begin{lemma}[Lemma 12 in \cite{matouvsek1996improved}]\label{entropy-tail-bound-lemma}
    Let $E$ be a real valued random variable satisfying the tail estimates
    \begin{equation}
        \mathbb{P}(E \geq \alpha M) \leq e^{-\alpha^2/2},~ \mathbb{P}(E \leq -\alpha M) \leq e^{-\alpha^2/2},
    \end{equation}
    for some parameter $M$. Let $b(E)$ denote the random variable obtained by rounding $E$ to the nearest multiple of $\Delta = 2\lambda M$. Then the entropy of $b$ satisfies
    \begin{equation}
        H(b) \leq G(\lambda) := C_0\begin{cases}
            e^{-\lambda^2/9} & \lambda \geq 10\\
            1 & .1 < \lambda < 10\\
            -\log(\lambda) & \lambda \leq .1
        \end{cases}
    \end{equation}
    for an absolute constant $C_0$.
\end{lemma}

To apply this lemma, we bound the tails of the random variables $E^m_{x,l}(\epsilon)$. This follows using Bernstein's inequality as in \cite{matouvsek1996improved}. Fix an $l\geq 2$ and an $x\in N_l$. We call an index $j$ `good' if
\begin{equation}\label{good-index-definition-zonotope}
    \sign(\omega\cdot x + b) = \sign(\omega\cdot \pi_{l-1}(x) = b)
\end{equation}
for $\theta := (\omega,b) = u_j,v_j,$ and $w_j$,
and `bad' otherwise. Using \eqref{phi-bound-zero} we see that for the good indices we have
\begin{equation}\label{good-index-bound}
    \Psi^m_{x,l,j} = 0.
\end{equation}

For the bad indices, we utilize \ref{bound-on-phi-nonzero} to get
\begin{equation}\label{bad-indices-bound}
    \|\Psi^m_{x,l,j}\| \leq C2^{-(k-m)l}N^{-1},
\end{equation}
since $a_{u_j},a_{v_j} \leq 2/N$ by \eqref{bound-on-coefficients-small-half}. Next, we bound the number of bad indices. An index is bad if
\begin{equation}\label{eq-391}
    \sign(\omega\cdot x + b) \neq \sign(\omega\cdot \pi_{l-1}(x) + b),
\end{equation}
for $\theta := (\omega,b) = u_j,v_j$ or $w_j$. From the construction of $N_{l-1}$ using Lemma \ref{covering-set-lemma}, the number of $\theta$ for which \eqref{eq-391} occurs (and thus the number of bad indices) is bounded by $2^{-l}|S_-| \leq C2^{-l}N$.

Thus, Bernstein's inequality (applied only to the bad indices) gives the following bound on the components of the random variable $E^m_{x,l}$,
\begin{equation}\label{tail-bound-estimate-466}
    \mathbb{P}\left((E^m_{x,l})_\textbf{i} \geq \alpha M_l^m \right) \leq e^{-\alpha^2/2}
\end{equation}
for all $\textbf{i}\in I^m$, where
\begin{equation}
    M_l^m = C2^{-\left(k-m+\frac{1}{2}\right)l}N^{-1/2},
\end{equation}
for a constant $C = C(d,k)$. The same bound holds also for the negative tails. 

The proof is completed via the following calculation (see \cite{matouvsek1996improved,matouvsek1996discrepancy} for similar calculations). We let $\alpha,\beta > 0$ be parameters to be specified in a moment and define a function
\begin{equation}
    \Lambda_{\alpha,\beta}(x) = \begin{cases}
        2^{\alpha x} & x\geq 0\\
        2^{\beta x} & x\leq 0.
    \end{cases}
\end{equation}
Let $\kappa > 0$ be another parameter and set $\tau$ to be the largest integer satisfying $2^{d\tau} \leq \kappa N$. We set the discretization parameter to be
\begin{equation}\label{equation-426}
    \Delta^m_l = 2M_l^m\Lambda_{\alpha,\beta}(l - \tau).
\end{equation}

Fix an $m$ with $0\leq m\leq k$. We calculate 
\begin{equation}
\begin{split}
     \sum_{l=1}^L \sum_{i=0}^{k-m} 2^{-il} \Delta^{m+i}_l &\leq CN^{-1/2}\sum_{i=0}^{k-m} \sum_{l=-\infty}^\infty 2^{-il}2^{-\left(k-m-i+\frac{1}{2}\right)l} \Lambda_{\alpha,\beta}(l - \tau)\\
     &\leq CN^{-1/2} \sum_{l=-\infty}^\infty 2^{-\left(k-m+\frac{1}{2}\right)l} \Lambda_{\alpha,\beta}(l - \tau),
\end{split}
\end{equation}
since all of the terms in the sum over $i$ are the same.
Making a change of variables in the last sum, we get
\begin{equation}
    \sum_{l=1}^L \sum_{i=0}^{k-m} 2^{-il} \Delta^{m+i}_l \leq CN^{-1/2}2^{-\left(k-m+\frac{1}{2}\right)\tau}\sum_{l=-\infty}^\infty 2^{-\left(k-m+\frac{1}{2}\right)l} \Lambda_{\alpha,\beta}(l). 
\end{equation}
If we now choose $\alpha$ and $\beta$ such that the above sum over $l$ converges (this will happen as long as $\alpha < k-m+\frac{1}{2} < \beta$), then we get
\begin{equation}
     \sum_{l=1}^L \sum_{i=0}^{k-m} 2^{-il} \Delta^{m+i}_l \leq CN^{-\frac{1}{2} - \frac{2(k-m)+1}{2d}},
\end{equation}
since by construction $2^\tau \geq (1/2)(\kappa N)^{1/d}$ (note the constant $C$ depends upon the choice of $\kappa$ which will be made shortly). This verifies \eqref{condition-on-delta-k}.

To verify the entropy condition, we calculate
\begin{equation}\label{G-entropy-bound}
    \sum_{l=1}^L\sum_{m=0}^k\sum_{x\in N_l} H(b^m_{x,l}(\epsilon)) \leq \sum_{l=1}^L\sum_{m=0}^k\sum_{x\in N_l}\sum_{\textbf{i}\in I^m} H(b^m_{x,l}(\epsilon)_{\textbf{i}})
\end{equation}
using subadditivity of the entropy. We now use Lemma \ref{entropy-tail-bound-lemma} combined with the tail bound estimate \eqref{tail-bound-estimate-466} to get
\begin{equation}
    H(b^m_{x,l}(\epsilon)_{\textbf{i}}) \leq G(\Delta_l^m/(2M_l^m)) = G(\Lambda_{\alpha,\beta}(l-\tau)).
\end{equation}
Using that $|I^m| \leq C = C(d,k)$ and $|N_l| \leq C2^{dl}$, we get that
\begin{equation}
    \sum_{l=1}^L\sum_{m=0}^k\sum_{x\in N_l} H(b^m_{x,l}(\epsilon)) \leq C\sum_{l=1}^L 2^{dl}G(\Lambda_{\alpha,\beta}(l-\tau)) \leq C\sum_{l=\infty}^\infty 2^{dl}G(\Lambda_{\alpha,\beta}(l-\tau)).
\end{equation}
Finally, making another change of variables, we get
\begin{equation}
    \sum_{l=1}^L\sum_{m=0}^k\sum_{x\in N_l} H(b^m_{x,l}(\epsilon)) \leq C2^{d\tau}\sum_{l=\infty}^\infty 2^{dl}G(\Lambda_{\alpha,\beta}(l)) \leq C\kappa N,
\end{equation}
since it is easy to verify that the above sum over $l$ converges for any $\alpha,\beta > 0$. Choosing $\kappa$ sufficiently small so that $C\kappa \leq 1/60$ will guarantee that the condition \eqref{entropy-sum-bound-354} is satisfied (since $t \geq N/12$ by \eqref{number-of-P-lower-bound}) and this completes the proof.
\end{proof}

\begin{proof}[Proof of Lemma \ref{relu-k-representation-lemma}]
        We first prove that for $m = 0,...,k$, $\theta := (\omega,b)\in S^{d-1}\times [-1,1]$ and $l=1,...,L$ we have
        \begin{equation}\label{sigma-l-decomposition-reluk}
        \sigma_k^{(m)}(x_l;\theta) = \sum_{j=1}^l\phi_{x_j,j}^m(\theta) + \sum_{i=1}^{k-m}\sum_{j=1}^{l-1}\left\langle \phi_{x_j,j}^{m+i}(\theta), \Gamma_{i,j}^{m,l}(x)\right\rangle,
    \end{equation}
    where the tensors $\Gamma_{i,j}^{m,l}(x)$ satisfy the bound
    \begin{equation}\label{bound-on-gamma-tensors}
        \left\|\Gamma_{i,j}^{m,l}(x)\right\| \leq C2^{-ij},
    \end{equation}
    for a constant $C = C(d,k)$. 

        We prove this by (reverse) induction on $m$. Note if $m=k$ the equation \eqref{sigma-l-decomposition-reluk} holds since the definition of $\phi_{x,l}^{(m)}$ in \eqref{definition-of-phi-m} becomes 
        \begin{equation}
        \phi_{x,l}^m(\theta) = \begin{cases}
            \sigma_k^{(m)}(x;\theta) - \sigma_k^{(m)}(\pi_{l-1}(x);\theta) & l \geq 2\\
            \sigma_k^{(m)}(x;\theta) & l=1,
        \end{cases}
        \end{equation}
        and so the sum in \eqref{sigma-l-decomposition-reluk} telescopes.
        
        Let $0\leq m \leq k$ and suppose that \eqref{sigma-l-decomposition-reluk} holds for $m+1,...,k$. We will show that it also holds for $m$. Expanding out the Taylor polynomial in the definition of $\phi_{x,l}^m$ for $x = x_l$ we see that
        \begin{equation}
        \begin{split}
            \sigma_k^{(m)}(x_l;\theta) &= \phi_{x_l,l}^m(\theta) + \mathcal{T}^{m,k-m}_{x_{l-1}}(x_l;\theta)\\ &= \phi_{x_l,l}^m(\theta) + \sum_{q=0}^{k-m} \frac{1}{q!}\left\langle\sigma_k^{(m+q)}(x_{l-1};\theta), (x_l - x_{l-1})^{\otimes q}\right\rangle\\
            &= \phi_{x_l,l}^m(\theta) + \sigma_k^{(m)}(x_{l-1};\theta) + \sum_{q=1}^{k-m} \frac{1}{q!}\left\langle\sigma_k^{(m+q)}(x_{l-1};\theta), (x_l - x_{l-1})^{\otimes q}\right\rangle.
        \end{split}            
        \end{equation}
        Applying this expansion recursively to $\sigma_k^{(m)}(x_{l-1};\theta)$, we get
        \begin{equation}
        \begin{split}
            \sigma_k^{(m)}(x_l;\theta) = \sum_{j=1}^l \phi_{x_j,j}^m(\theta) + \sum_{p=1}^{l-1}\sum_{q=1}^{k-m}\frac{1}{q!}\left\langle\sigma_k^{(m+q)}(x_{p};\theta), (x_{p+1} - x_{p})^{\otimes q}\right\rangle.
        \end{split}
        \end{equation}
        Now we use the inductive assumption to expand $\sigma_k^{(m+q)}(x_{p};\theta)$ using \eqref{sigma-l-decomposition-reluk} and apply the identity \eqref{contraction-tensor-product-identity} to get
        \begin{equation}
        \begin{split}
            \sigma_k^{(m)}(x_l;\theta) = \sum_{j=1}^l \phi_{x_j,j}^m(\theta) &+ \sum_{p=1}^{l-1}\sum_{q=1}^{k-m}\frac{1}{q!}\sum_{j=1}^{p}\left\langle \phi_{x_j,j}^{m+q}(\theta), (x_{p+1} - x_{p})^{\otimes q}\right\rangle\\
            & + \sum_{p=1}^{l-1}\sum_{q=1}^{k-m}\frac{1}{q!}\sum_{i'=1}^{k-m-q}\sum_{j=1}^{p-1} \left\langle \phi_{x_j,j}^{m+q+i'}(\theta),\Gamma_{i',j}^{m+q,l}(x) \otimes (x_{p+1} - x_p)^{\otimes q}\right\rangle.
        \end{split}
        \end{equation}
        Rearranging this sum, we obtain
        \begin{equation}
            \sigma_k^{(m)}(x_l;\theta) = \sum_{j=1}^l\phi_{x_j,j}^m(\theta) + \sum_{i=1}^{k-m}\sum_{j=1}^{l-1}\left\langle \phi_{x_j,j}^{m+i}(\theta), \Gamma_{i,j}^{m,l}(x)\right\rangle,
        \end{equation}
        where the tensors $\Gamma_{i,j}^{m,l}(x)$ are defined recursively by
        \begin{equation}
            \Gamma_{i,j}^{m,l}(x) = \frac{1}{i!}\sum_{p=j}^{l-1} (x_{p+1} - x_{p})^{\otimes i} + \sum_{p=j}^{l-1} \sum_{q=1}^{i-1}\frac{1}{q!} \Gamma_{i-q,j}^{m+q,l}(x)\otimes (x_{p+1} - x_{p})^{\otimes q}.
        \end{equation}
        Finally, we bound the norm $\|\Gamma_{i,j}^{m,l}(x)\|$. By construction, the points $x_p$ satisfy $|x_{p+1} - x_p| \leq C2^{-p}$ for a dimension dependent constant $C = C(d)$. This gives the bound
        \begin{equation}
            \|\Gamma_{i,j}^{m,l}(x)\| \leq C\left(\frac{1}{i!}\sum_{p=j}^{l-1}2^{-pj} + \sum_{p=j}^{l-1} \sum_{q=1}^{i-1}\frac{1}{q!} 2^{-pq}\|\Gamma_{i-q,j}^{m+q,l}\|\right).
        \end{equation}
        Using the inductive assumption to bound $\|\Gamma_{i-q,j}^{m+q,l}\|$ we get
        \begin{equation}
            \begin{split}
                \|\Gamma_{i,j}^{m,l}\| &\leq C\left(\frac{1}{i!}\sum_{p=j}^{l-1}2^{-pj} + \sum_{p=j}^{l-1} \sum_{q=1}^{i-1}\frac{1}{q!} 2^{-pq}2^{-(i-q)j}\right)\\
                &\leq C\left(\frac{1}{i!}\sum_{p=j}^{\infty}2^{-pj} + 2^{-ij}\sum_{p=j}^{\infty} \sum_{q=1}^{\infty}\frac{1}{q!} 2^{-q(p-j)}\right)\leq C2^{-ij},
            \end{split}
        \end{equation}
        for a constant $C$ which may be different for each value of $m$. Since there are only $k+1$ different values of $m$, the constant $C = C(d,k)$ can be taken uniform in $m$. This proves \eqref{sigma-l-decomposition-reluk}.

        To prove \eqref{sigma-decomposition-reluk}, we write
        \begin{equation}\label{decomposition-x-610}
        \begin{split}
            \sigma_k^{(m)}(x;\theta) &= \left[\sigma_k^{(m)}(x;\theta) - \mathcal{T}^{m,k-m}_{x_L}(x;\theta)\right] + \mathcal{T}^{m,k-m}_{x_L}(x;\theta)\\ 
            &= \left[\sigma_k^{(m)}(x;\theta) - \mathcal{T}^{m,k-m}_{x_L}(x;\theta)\right] + \sum_{q=0}^{k-m} \frac{1}{q!}\left\langle\sigma_k^{(m+q)}(x_L;\theta), (x - x_L)^{\otimes q}\right\rangle.
        \end{split}
        \end{equation}
        We claim that if $\theta\in S_-$, then the first term
        \begin{equation}\label{equation-364-difference-0}
        \sigma_k^{(m)}(x;\theta) - \mathcal{T}^{m,k-m}_{x_L}(x;\theta) = 0.
        \end{equation}
        This follows since by construction using Lemma \ref{covering-set-ball-lemma}, we have the bound
        \begin{equation}
            |S_-\cap \{(\omega,b)\in S^{d-1}\times [-1,1],~\sign{(\omega\cdot x + b)} \neq \sign{(\omega\cdot \pi_L(x) + b)}\}| \leq 2^{-L} |S_-| < 1,
        \end{equation}
        since $2^{-L} < N^{-1}$. Thus for all $\theta := (\omega,b)\in S_-$ and $x\in S^d$, we have $\sign{(\omega\cdot x + b)} = \sign{(\omega\cdot \pi_L(x) + b)}$, and the argument following \eqref{phi-bound-zero} implies that the difference in \eqref{equation-364-difference-0} vanishes. 
        
        Next, we expand each term in the sum in \eqref{decomposition-x-610} using \eqref{sigma-l-decomposition-reluk} with $l = L$ to get (using again the identity \eqref{contraction-tensor-product-identity})
        \begin{equation}
            \sigma_k^{(m)}(x;\theta) = \sum_{q=0}^{k-m}\frac{1}{q!}\sum_{j=1}^L\left\langle\phi_{x_j,j}^{m+q}(\theta), (x - x_L)^{\otimes q}\right\rangle + \sum_{q=0}^{k-m}\frac{1}{q!}\sum_{i'=1}^{k-m-q}\sum_{j=1}^{L-1}\left\langle \phi_{x_j,j}^{m+q+i'}(\theta), \Gamma_{i',j}^{m+q,L}\otimes (x-x_L)^{\otimes q}\right\rangle.
        \end{equation}
        Rewriting this, we get
        \begin{equation}
            \sigma_k^{(m)}(x;\theta) = \sum_{j=1}^L\phi^{m}_{x_j,j}(\theta) + \sum_{i=1}^{k-m}\sum_{j=1}^L  \left\langle\phi^{m+i}_{x_j,j}(\theta),\Gamma^{m}_{i,j}(x)\right\rangle,
        \end{equation}
        where the tensors $\Gamma^{m}_{i,j}(x)$ are given by
        \begin{equation}
            \Gamma^{m}_{i,j}(x) = \frac{1}{i!}(x - x_L)^{\otimes i} + \sum_{q=0}^{i-1}\Gamma_{i-q,j}^{m+q,L}\otimes (x-x_L)^{\otimes q}.
        \end{equation}
        Finally, we bound the norm $\|\Gamma^{m}_{i,j}(x)\|$. Using that $|x - x_L| \leq C2^{-L}$ and the bound \eqref{bound-on-gamma-tensors} we get
        \begin{equation}
            \|\Gamma^{m}_{i,j}(x)\| \leq \frac{2^{-Li}}{i!} + C\sum_{q=0}^{i-1} 2^{-(i-q)j}2^{-Lq} \leq C2^{-ij},
        \end{equation}
        for a constant $C(d,k)$, since $j\leq L$ and $0\leq i \leq k$. Upon relabelling $j$ to $l$ this is exactly Lemma \ref{relu-k-representation-lemma}.
    \end{proof}

\section{Acknowledgements}
We would like to thank Ron DeVore, Rob Nowak, Jinchao Xu, and Rahul Parhi for helpful discussions. This work was supported by the National Science Foundation (DMS-2424305 and CCF-2205004) as well as the MURI ONR grant N00014-20-1-2787.

\bibliographystyle{amsplain}
\bibliography{refs}

\providecommand{\bysame}{\leavevmode\hbox to3em{\hrulefill}\thinspace}
\providecommand{\MR}{\relax\ifhmode\unskip\space\fi MR }
\providecommand{\MRhref}[2]{%
  \href{http://www.ams.org/mathscinet-getitem?mr=#1}{#2}
}
\providecommand{\href}[2]{#2}
\begin{thebibliography}{10}

\bibitem{alon2016probabilistic}
Noga Alon and Joel~H Spencer, \emph{The probabilistic method}, John Wiley \&
  Sons, 2016.

\bibitem{bach2017breaking}
Francis Bach, \emph{Breaking the curse of dimensionality with convex neural
  networks}, The Journal of Machine Learning Research \textbf{18} (2017),
  no.~1, 629--681.

\bibitem{barron1992neural}
Andrew~R Barron, \emph{Neural net approximation}, Proc. 7th Yale workshop on
  adaptive and learning systems, vol.~1, 1992, pp.~69--72.

\bibitem{barron1993universal}
\bysame, \emph{Universal approximation bounds for superpositions of a sigmoidal
  function}, IEEE Transactions on Information theory \textbf{39} (1993), no.~3,
  930--945.

\bibitem{barron2008approximation}
Andrew~R Barron, Albert Cohen, Wolfgang Dahmen, and Ronald~A DeVore,
  \emph{Approximation and learning by greedy algorithms}, The Annals of
  Statistics (2008), 64--94.

\bibitem{bartlett2019nearly}
Peter~L Bartlett, Nick Harvey, Christopher Liaw, and Abbas Mehrabian,
  \emph{Nearly-tight vc-dimension and pseudodimension bounds for piecewise
  linear neural networks}, Journal of Machine Learning Research \textbf{20}
  (2019), no.~63, 1--17.

\bibitem{beck1984some}
J{\'o}zsef Beck, \emph{Some upper bounds in the theory of irregularities of
  distribution}, Acta Arithmetica \textbf{43} (1984), 115--130.

\bibitem{betke1983estimating}
U~Betke and P~McMullen, \emph{Estimating the sizes of convex bodies from
  projections}, Journal of the London Mathematical Society \textbf{2} (1983),
  no.~3, 525--538.

\bibitem{bourgain1988distribution}
J~Bourgain and J~Lindenstrauss, \emph{Distribution of points on spheres and
  approximation by zonotopes}, Israel Journal of Mathematics \textbf{64}
  (1988), 25--31.

\bibitem{bourgain1993approximating}
\bysame, \emph{Approximating the ball by a minkowski sum of segments with equal
  length}, Discrete \& computational geometry \textbf{9} (1993), no.~2,
  131--144.

\bibitem{bourgain1989approximation}
J~Bourgain, J~Lindenstrauss, and V~Milman, \emph{Approximation of zonoids by
  zonotopes}, Acta Mathematica \textbf{162} (1989), no.~1, 73--141.

\bibitem{cheang2000better}
Gerald~HL Cheang and Andrew~R Barron, \emph{A better approximation for balls},
  Journal of Approximation Theory \textbf{104} (2000), no.~2, 183--203.

\bibitem{devore2021neural}
Ronald DeVore, Boris Hanin, and Guergana Petrova, \emph{Neural network
  approximation}, Acta Numerica \textbf{30} (2021), 327--444.

\bibitem{devore1998nonlinear}
Ronald~A DeVore, \emph{Nonlinear approximation}, Acta numerica \textbf{7}
  (1998), 51--150.

\bibitem{devore1996some}
Ronald~A DeVore and Vladimir~N Temlyakov, \emph{Some remarks on greedy
  algorithms}, Advances in computational Mathematics \textbf{5} (1996),
  173--187.

\bibitem{donahue1997rates}
Michael~J Donahue, Christian Darken, Leonid Gurvits, and Eduardo Sontag,
  \emph{Rates of convex approximation in non-hilbert spaces}, Constructive
  Approximation \textbf{13} (1997), 187--220.

\bibitem{dudley1978central}
Richard~M Dudley, \emph{Central limit theorems for empirical measures}, The
  Annals of Probability (1978), 899--929.

\bibitem{ma2022barron}
Weinan E, Chao Ma, and Lei Wu, \emph{The barron space and the flow-induced
  function spaces for neural network models}, Constructive Approximation
  \textbf{55} (2022), no.~1, 369--406.

\bibitem{evans2010partial}
Lawrence~C Evans, \emph{Partial differential equations}, vol.~19, American
  Mathematical Soc., 2010.

\bibitem{haussler1995sphere}
David Haussler, \emph{Sphere packing numbers for subsets of the {B}oolean
  n-cube with bounded {V}apnik-{C}hervonenkis dimension}, Journal of
  Combinatorial Theory, Series A \textbf{69} (1995), no.~2, 217--232.

\bibitem{jones1992simple}
Lee~K Jones, \emph{A simple lemma on greedy approximation in hilbert space and
  convergence rates for projection pursuit regression and neural network
  training}, The annals of Statistics (1992), 608--613.

\bibitem{joos2023isoperimetric}
Antal Jo{\'o}s and Zsolt L{\'a}ngi, \emph{Isoperimetric problems for
  zonotopes}, Mathematika \textbf{69} (2023), no.~2, 508--534.

\bibitem{klusowski2018approximation}
Jason~M Klusowski and Andrew~R Barron, \emph{Approximation by combinations of
  relu and squared relu ridge functions with $\ell^1$ and $\ell^0$ controls},
  IEEE Transactions on Information Theory \textbf{64} (2018), no.~12,
  7649--7656.

\bibitem{linhart1989approximation}
J~Linhart, \emph{Approximation of a ball by zonotopes using uniform
  distribution on the sphere}, Archiv der Mathematik \textbf{53} (1989),
  82--86.

\bibitem{lu2022priori}
Jianfeng Lu and Yulong Lu, \emph{A priori generalization error analysis of
  two-layer neural networks for solving high dimensional schr{\"o}dinger
  eigenvalue problems}, Communications of the American Mathematical Society
  \textbf{2} (2022), no.~1, 1--21.

\bibitem{ma2022uniform}
Limin Ma, Jonathan~W Siegel, and Jinchao Xu, \emph{Uniform approximation rates
  and metric entropy of shallow neural networks}, Research in the Mathematical
  Sciences \textbf{9} (2022), no.~3, 46.

\bibitem{makovoz1996random}
Yuly Makovoz, \emph{Random approximants and neural networks}, Journal of
  Approximation Theory \textbf{85} (1996), no.~1, 98--109.

\bibitem{mao2024approximation}
Tong Mao, Jonathan~W Siegel, and Jinchao Xu, \emph{Approximation rates for
  shallow {ReLU$^k$} neural networks on sobolev spaces via the radon
  transform}, arXiv preprint arXiv:2408.10996 (2024).

\bibitem{matousek1991cutting}
Ji{\v{r}}{\'\i} Matou{\v{s}}ek, \emph{Cutting hyperplane arrangements},
  Discrete \& Computational Geometry \textbf{6} (1991), 385--406.

\bibitem{matouvsek1995tight}
\bysame, \emph{Tight upper bounds for the discrepancy of half-spaces}, Discrete
  \& Computational Geometry \textbf{13} (1995), 593--601.

\bibitem{matouvsek1996improved}
\bysame, \emph{Improved upper bounds for approximation by zonotopes}, Acta
  Mathematica \textbf{177} (1996), 55--73.

\bibitem{matousek1999geometric}
\bysame, \emph{Geometric discrepancy: An illustrated guide}, vol.~18, Springer
  Science \& Business Media, 1999.

\bibitem{matouvsek1996discrepancy}
Ji{\v{r}}{\'\i} Matou{\v{s}}ek and Joel Spencer, \emph{Discrepancy in
  arithmetic progressions}, Journal of the American Mathematical Society
  \textbf{9} (1996), no.~1, 195--204.

\bibitem{ongie2019function}
Greg Ongie, Rebecca Willett, Daniel Soudry, and Nathan Srebro, \emph{A function
  space view of bounded norm infinite width relu nets: The multivariate case},
  International Conference on Learning Representations, 2019.

\bibitem{parhi2021banach}
Rahul Parhi and Robert~D Nowak, \emph{Banach space representer theorems for
  neural networks and ridge splines}, The Journal of Machine Learning Research
  \textbf{22} (2021), no.~1, 1960--1999.

\bibitem{pisier1981remarques}
Gilles Pisier, \emph{Remarques sur un r{\'e}sultat non publi{\'e} de b.
  maurey}, S{\'e}minaire d'Analyse fonctionnelle (dit" Maurey-Schwartz")
  (1981), 1--12.

\bibitem{rockafellar1997convex}
R~Tyrrell Rockafellar, \emph{Convex analysis}, vol.~11, Princeton university
  press, 1997.

\bibitem{siegel2023optimal}
Jonathan~W Siegel, \emph{Optimal approximation rates for deep relu neural
  networks on sobolev and besov spaces}, Journal of Machine Learning Research
  \textbf{24} (2023), no.~357, 1--52.

\bibitem{siegel2023greedy}
Jonathan~W Siegel, Qingguo Hong, Xianlin Jin, Wenrui Hao, and Jinchao Xu,
  \emph{Greedy training algorithms for neural networks and applications to
  pdes}, Journal of Computational Physics \textbf{484} (2023), 112084.

\bibitem{siegel2020approximation}
Jonathan~W Siegel and Jinchao Xu, \emph{Approximation rates for neural networks
  with general activation functions}, Neural Networks \textbf{128} (2020),
  313--321.

\bibitem{siegel2022high}
\bysame, \emph{High-order approximation rates for shallow neural networks with
  cosine and reluk activation functions}, Applied and Computational Harmonic
  Analysis \textbf{58} (2022), 1--26.

\bibitem{siegel2022sharp}
\bysame, \emph{Sharp bounds on the approximation rates, metric entropy, and
  n-widths of shallow neural networks}, Foundations of Computational
  Mathematics (2022), 1--57.

\bibitem{siegel2023characterization}
\bysame, \emph{Characterization of the variation spaces corresponding to
  shallow neural networks}, Constructive Approximation (2023), 1--24.

\bibitem{spencer1985six}
Joel Spencer, \emph{Six standard deviations suffice}, Transactions of the
  American Mathematical Society \textbf{289} (1985), no.~2, 679--706.

\bibitem{temlyakov2008greedy}
Vladimir Temlyakov, \emph{Greedy approximation}, Acta Numerica \textbf{17}
  (2008), 235--409.

\bibitem{vapnik1999nature}
Vladimir Vapnik, \emph{The nature of statistical learning theory}, Springer
  science \& business media, 1999.

\bibitem{vapnik1971uniform}
VN~Vapnik and A~Ya Chervonenkis, \emph{On the uniform convergence of relative
  frequencies of events to their probabilities}, Theory of Probability and its
  Applications \textbf{16} (1971), no.~2, 264.

\bibitem{xu2020finite}
Jinchao Xu, \emph{Finite neuron method and convergence analysis},
  Communications in Computational Physics \textbf{28} (2020), no.~5,
  1707--1745.

\bibitem{yang2023optimal}
Yunfei Yang and Ding-Xuan Zhou, \emph{Optimal rates of approximation by shallow
  {ReLU}$^k$ neural networks and applications to nonparametric regression},
  Constructive Approximation (2024), 1--32.

\bibitem{yukich1995sup}
Joseph~E Yukich, Maxwell~B Stinchcombe, and Halbert White, \emph{Sup-norm
  approximation bounds for networks through probabilistic methods}, IEEE
  Transactions on Information Theory \textbf{41} (1995), no.~4, 1021--1027.

\end{thebibliography}
\end{document}